\documentclass[twoside]{article}

\usepackage[accepted]{aistats2020}
\usepackage[utf8]{inputenc} 
\usepackage{url}            
\usepackage{booktabs}       
\usepackage{amsfonts}       
\usepackage{nicefrac}       
\usepackage{microtype}      
\usepackage[round]{natbib}
\usepackage{verbatim}
\usepackage{tabu}

\usepackage{algorithm}
\usepackage{algorithmic}

\usepackage{amsmath}
\usepackage{amssymb}
\usepackage{amsthm}
\usepackage{bm}
\usepackage{bbm}
\usepackage{setspace}
\usepackage{xspace}
\usepackage{xcolor} 
\usepackage{wrapfig}
\usepackage{array} 
\usepackage{multirow}
\usepackage{enumitem}
\usepackage{stmaryrd}
\usepackage{mdframed}

\newcommand{\alinea}{\hspace{1em}\textbullet\ }
\usepackage{thm-restate}
\usepackage{caption}
\usepackage{subcaption}

\usepackage[colorinlistoftodos,disable, textwidth=20mm]{todonotes}

\definecolor{blued}{RGB}{70,197,221}
\definecolor{pearOne}{HTML}{2C3E50}
\definecolor{pearTwo}{HTML}{A9CF54}
\definecolor{pearTwoT}{HTML}{C2895B}
\definecolor{pearThree}{HTML}{E74C3C}
\colorlet{titleTh}{pearOne}
\colorlet{bull}{pearTwo}
\definecolor{pearcomp}{HTML}{B97E29}
\definecolor{pearFour}{HTML}{588F27}
\definecolor{pearFith}{HTML}{ECF0F1}
\definecolor{pearDark}{HTML}{2980B9}
\definecolor{pearDarker}{HTML}{1D2DEC}
\usepackage{hyperref}       
\hypersetup{
	colorlinks,
	citecolor=pearDark,
	linkcolor=pearThree,
	breaklinks=true,
	urlcolor=pearDarker}

\usepackage[most]{tcolorbox}

\definecolor{citrine}{rgb}{0.89, 0.82, 0.04}

\definecolor{graphicbackground}{rgb}{0.96,0.96,0.8}
\definecolor{rouge1}{RGB}{226,0,38}  
\definecolor{orange1}{RGB}{243,154,38}  
\definecolor{jaune}{RGB}{254,205,27}  
\definecolor{blanc}{RGB}{255,255,255} 
\definecolor{rouge2}{RGB}{230,68,57}  
\definecolor{orange2}{RGB}{236,117,40}  
\definecolor{taupe}{RGB}{134,113,127} 
\definecolor{gris}{RGB}{91,94,111} 
\definecolor{bleu1}{RGB}{38,109,131} 
\definecolor{bleu2}{RGB}{28,50,114} 
\definecolor{vert1}{RGB}{133,146,66} 
\definecolor{vert3}{RGB}{20,200,66} 
\definecolor{vert2}{RGB}{157,193,7} 
\definecolor{darkyellow}{RGB}{233,165,0}  
\definecolor{lightgray}{rgb}{0.9,0.9,0.9}
\definecolor{darkgray}{rgb}{0.6,0.6,0.6}
\definecolor{babyblue}{rgb}{0.54, 0.81, 0.94}
\definecolor{citrine}{rgb}{0.89, 0.82, 0.04}
\definecolor{misogreen}{rgb}{0.25,0.6,0.0}

\DeclareMathOperator*{\argmax}{arg\,max}
\DeclareMathOperator*{\argmin}{arg\,min}

\let\originalleft\left
\let\originalright\right
\renewcommand{\left}{\mathopen{}\mathclose\bgroup\originalleft}
\renewcommand{\right}{\aftergroup\egroup\originalright}

\newtheorem{assumption}{Assumption}
 \newtheorem{lemma}{Lemma}
 \newtheorem{theorem}{Theorem}
 \newtheorem{definition}{Definition}
 \newtheorem{corollary}{Corollary}

 \newtheorem{proposition}{Proposition}

\newcommand{\EE}[1]{\mathbb{E}\left[#1\right]}

\newcommand{\CommaBin}{\mathbin{\raisebox{0.5ex}{,}}}

\renewcommand{\epsilon}{\varepsilon}

\renewcommand{\tilde}{\widetilde}
\renewcommand{\bar}{\overline}

\newcommand{\nothere}[1]{}

\usepackage{xspace}
\renewcommand{\ttdefault}{lmtt}

\newcommand{\kometo}{\texttt{Kometo}\xspace}

\newcommand{\GPUCB}{\texttt{GP-UCB}\xspace}

\newcommand{\StoSOO}{\texttt{StoSOO}\xspace}
\newcommand{\POO}{\texttt{POO}\xspace}
\newcommand{\DOO}{\texttt{DOO}\xspace}
\newcommand{\SOO}{\texttt{SOO}\xspace}
\newcommand{\Zooming}{\texttt{Zooming}\xspace}

\newcommand{\MFPDOO}{\texttt{MFPDOO}\xspace}

\newcommand{\SequOOL}{\texttt{SequOOL}\xspace}
\newcommand{\StroquOOL}{\texttt{StroquOOL}\xspace}
\newcommand{\GPO}{\texttt{GPO}\xspace}

\newcommand{\HOO}{\texttt{HOO}\xspace}

\usepackage{authblk}

\AfterEndEnvironment{assumption}{\noindent\ignorespaces}
\AfterEndEnvironment{proposition}{\noindent\ignorespaces}
\AfterEndEnvironment{theorem}{\noindent\ignorespaces}
\AfterEndEnvironment{restatable}{\noindent\ignorespaces}
\AfterEndEnvironment{lemma}{\noindent\ignorespaces}
\AfterEndEnvironment{remark}{\noindent\ignorespaces}

\newcommand{\hmax}{\tilde{\Lambda}}

\newcommand{\jmax}{j_{\rm max}}
\newcommand{\invalpha}{\frac{1}{\alpha}}
\newcommand{\logrho}{\log\frac{1}{\rho}}

\newcommand{\tree}{\mathcal{T}}
\newcommand{\regretl}{r_\Lambda}
\newcommand{\htilde}{\tilde{h}}
\newcommand{\dalpharho}{(d+\invalpha )\logrho}
\newcommand{\abnu}{a_{b,\nu}}
\newcommand{\invbeta}{\frac{1}{\beta}}

\newcommand{\xstar}{x^\star}
\newcommand{\ncell}{N}
\newcommand{\bigO}{\mathcal{O}}
\newcommand{\bigOtilde}{\tilde{\mathcal{O}}}
\newcommand{\targetf}{f}
\newcommand{\Zeta}{Z}
\newcommand{\functionspace}{\mathcal{X}}
\newcommand{\finalvalue}{x_\Lambda}
\newcommand{\assumf}{Assumption~\ref{ass:opt}\xspace}
\newcommand{\assumsa}{Assumption~\ref{ass:smo}\ref{assumsa}\xspace}
\newcommand{\assumsb}{Assumption~\ref{ass:smo}\ref{assumsb}\xspace}
\newcommand{\assumsc}{Assumption~\ref{ass:smo}\ref{assumsc}\xspace}
\newcommand{\assumsymbol}{{\rm Asm}}
\newcommand{\optimum}{x^\star}
\newcommand{\policysymbol}{\pi}
\newcommand{\regretpolicy}{\regretl(\policysymbol)}
\newcommand{\fidfamily}{(f_z)_{z\in \Zeta}}
\newcommand{\assumset}{F(\assumsymbol,\targetf,\lambda)}
\newcommand{\partitioning}{\mathcal{P}}
\newcommand{\smoothset}{S(\partitioning, \nu,\rho,d,C)}
\newcommand{\partitioninghi}{\partitioning_{h,i}}
\newcommand{\trunctree}{T}
\newcommand{\trunctreea}{\trunctree_a}
\newcommand{\trunctreei}{\trunctree_i}
\newcommand{\trunctreeh}{\trunctreea^h}
\newcommand{\targettrunctree}{f^{\trunctreea}}
\newcommand{\fidtrunctree}{\targettrunctree_z}
\newcommand{\pointfidl}{X^{\trunctreea}_l}
\newcommand{\setfid}{S}
\newcommand{\setfidh}{\setfid^{\trunctreea}_h}
\newcommand{\trunctreeabis}{\trunctreea'}
\newcommand{\pointfidlbis}{X^{\trunctreeabis}_l\trunctree}
\newcommand{\setfidhbis}{S^{\trunctreeabis}_h}
\newcommand{\lsumi}{\sum\limits_{i=0}^{K^{s+1}-1}}
\newcommand{\lsuml}{\sum\limits_{l\in\mathbb{N}}}
\newcommand{\cmin}{C_{\min}}
\newcommand{\dmax}{d_{\max}}
\newcommand{\lambdac}{\tilde{\Lambda}_c}

\begin{document}

%

%

\renewcommand{\ttdefault}{lmtt}
\twocolumn[

\aistatstitle{Adaptive multi-fidelity optimization with fast learning rates}

\aistatsauthor{C\^{o}me Fiegel \And Victor Gabillon \And Michal Valko}

\aistatsaddress{\'Ecole Normale Sup\'erieure, Paris\\Inria Lille\And Huawei R\&D, UK\And Inria Lille }


]
\begin{abstract}
 In \textit{multi-fidelity optimization}, we have access to \textit{biased} approximations of varying costs of the target function. In this work, we study the setting of optimizing a locally smooth function with a limited budget $\Lambda$, where the learner has to make a trade-off between the cost and the bias of these approximations. We first prove lower bounds for the \textit{simple regret} under different assumptions on the fidelities, based on a \textit{cost-to-bias} function. We then present the \kometo algorithm which achieves, with additional logarithmic factors, the same rates \textit{without} any knowledge of the function smoothness and fidelity assumptions and improving prior results. Finally, we empirically show that our algorithm outperforms prior multi-fidelity optimization methods without the knowledge of problem-dependent parameters.
\end{abstract}


\section{Introduction}

In multi-fidelity optimization~\citep{cutler2014reinforcement,huang2006sequential,kandasamy2016multi,kandasamy2017multi}, the learner \textit{actively} optimizes a function but only observes, at each of the rounds, biased values of that function. The learner can \emph{pay} to reduce the bias of the observed function values. The smaller the bias the higher the cost, urging the learner to carefully allocate its total cost budget $\Lambda$ on the fly.

We consider the case of \textit{derivative-free optimization} where no gradient information is available~\citep{matyas1965random}. 
 This is of great interest for the multiple applications  in which it is either difficult to access, compute, or even define gradients \citep{nesterov2017random}.
 Using only zero-order information, derivative-free optimization addresses optimising over functions that are  not differentiable, non-continuous, or  non-smooth. Moreover, there are known  methods that work without knowing the smoothness parameters $(\nu,\rho)$ of the function \citep{auer2007improved,kleinberg2008multi,grill2015black-box,valko2013stochastic,bartlett2019simple}. 

Derivative-free multi-fidelity optimization is useful in particular for hyper-parameter tuning of \textit{complex} machine learning models, where each evaluation of the model is costly such as tokamak simulators.
However, the mapping between the hyper-parameter and performance of the learned model can be highly non-convex and non-smooth.
Moreover, training a model, given the hyper-parameters can be  expensive and time-consuming \citep{sen2018}. In a situation, where computation or time are constrained by a budget, these constraints prevent us from carefully evaluating the qualities of all the models generated from a continuous set of hyper-parameters.
Then,  given one fixed set of hyper-parameters,  the bias  of the estimation of the quality of fully-trained model is a (decreasing) function~$\zeta$ of the amount of computation resource spent training the model.
Ultimately, we would expect this bias to be zero if the model is trained until convergence. However, the bias function $\zeta$ is a function that depends on the type of trained models and that is in applications a priori unknown.

The most related approach for the considered setting is the \MFPDOO algorithm of
\citet{sen2018}. In order to provide theoretical guarantees for \MFPDOO, $\zeta$  is either assumed to be known or some parametric assumptions on $\zeta$ are made and the parameters are estimated online. However, knowing $\zeta$ or its parametric family is unrealistic.

In this paper, we 
propose a new method called \kometo that adapts to the unknown $\zeta$ and the unknown smoothness parameters $(\nu,\rho)$. 
Our analysis is more general than the analysis of \cite{sen2018} and provides a broader and finer set of behaviors of the cost-to-bias function. This allows us to provide a characterisation of the complexity of the problem by providing  the first regret lower bounds in multi-fidelity optimization. We also show that \kometo obtains rates that match the ones of our lower bounds and improves upon the rates of \MFPDOO \textit{while dropping the assumptions of knowing the bias function $\zeta$ in advance.}

\paragraph{Related work} 
Among the large work on derivative-free optimization, we focus on algorithms that perform well under \emph{minimal} assumptions as well as minimal knowledge of the function. 
Under 	\emph{weak}/\emph{local} smoothness around one global maximum \citep{auer2007improved,kleinberg2008multi,bubeck2011x},  
 some algorithms require the knowledge of the local smoothness such as \HOO \citep{bubeck2011x}, \Zooming~\citep{kleinberg2008multi}, or \DOO \citep{munos2011optimistic}. Among the work relying on an \emph{unknown} local  smoothness,
	 \SequOOL~\citep{bartlett2019simple} improves on \SOO \citep{munos2011optimistic,kawaguchi2016global} 
	 and represents the state-of-the-art for the deterministic feedback.
	For the stochastic feedback, \StoSOO \citep{valko2013stochastic} extends \SOO for a limited class of functions. \POO~\citep{grill2015black-box} and \GPO~\citep{shang2019general} provide more general results. Finally,
	\StroquOOL~\citep{bartlett2019simple} matches, up to log factors, the  guarantees of \SequOOL and \GPO for deterministic and stochastic feedback respectively, \textit{without requiring the knowledge of the range of the noise $b$.}
	
Multi-fidelity optimization is a well studied setting. Here, we address  \textit{online} multi-fidelity optimization. Many approaches rely on Bayesian models, e.g., Gaussian processes. \citet{zhang2019} relies on entropic search to find the maximum, while \citet{kanda2016mf} adapts \GPUCB \citep{srinivas2010gaussian} to multi-fidelity setting. Most of these methods need an access to a bias function, while \citet{ghosh2019} use the cost of the approximations to estimate its values.
\cite{li2017hyperband} obtains good empirical results by trying a lot of configurations at low fidelities and progressively eliminating the less interesting ones while using higher and higher fidelities.
Two prior works adapted algorithms working under local smoothness around one global maximum to multi-fidelity settings. First, \cite{sen2018} adapted  \POO \citep{grill2015black-box} to deterministic multi-fidelity settings and later \cite{sen2019} made it work under stochastic ones.

\paragraph{Main contributions}
\begin{itemize}
    \item We give more general assumptions on the fidelity approximations based on their cost while keeping the smoothness assumption on the target function.
    \item We prove lower bounds of the simple regret under these more general and different assumptions.
    \item We  provide \kometo, an algorithm that, in deterministic settings, without \textit{any} knowledge on the bias function and the smoothness of the target function, achieves minimax optimal rates for simple regret up to logarithmic factors on all considered assumptions on the fidelities. It improves the previously proven guaranteed rates under local smoothness assumptions of \cite{sen2018}, except in the case $\alpha=1$ of \assumsa,\footnote{hyperbolic decreasing of the cost-to-bias function} where it has  additional logarithmic factors. Our \kometo comes with important properties:
    \begin{itemize}
    \item  It does \textit{not} assume an access to the target function, only an access to increasingly better approximations, unlike previous algorithms as the ones of \cite{sen2018}.
    \item   It only uses the comparisons of evaluations at the same fidelity level, and not directly the values of the evaluations, which leads to  \textit{weaker fidelity assumptions and better empirical results.}
    \item   It works in stochastic settings by changing the number of evaluations at higher fidelities.
        \end{itemize}
    \item We provide synthetic experiments and a hyper-parameter tuning experiment to demonstrate the efficiency of  \kometo.
\end{itemize}
\section{Problem setting}
In this section, we introduce a generalization of the settings presented by \cite{sen2018}.

We want to optimize a target function $\targetf: \functionspace \rightarrow \mathbb{R}$ under a budget $\Lambda\in \mathbb{R_+}$. The evaluation of this target function is done through its fidelity approximations. We thus denote by $\Zeta=\left[0,1\right]$ the fidelity space and by $(f_z)_{z\in Z}$ the fidelity approximations. In particular, $z=0$ corresponds to the lowest fidelity, while $z=1$ corresponds to the highest one. We also denote by $\zeta : \Zeta \rightarrow \mathbb{\bar{R_+}}$ the unknown bias function, such that there exists a family $(g_z)_{z\in \Zeta}$ of real-valued strictly increasing function with $\Vert \targetf - g_z \circ f_{z} \Vert_{\infty}\leq\zeta(z)$ for $z \in \Zeta;$ motivations for this assumption are explained below.
A known cost function $\lambda : \Zeta \rightarrow \mathbb{\bar{R_+}}$ indicates the budget used at each evaluation for a given fidelity.
We also assume that the algorithm can request, for any $c\geq 1$, a fidelity $z_c$ such that $\lambda (z_c)\leq c$, and we define $\Phi:[1,+\infty[ \to \bar{\mathbb{R}_+}$ with $\Phi(c)=\zeta (z_c)$, the cost-to-bias function, which gives for each cost $c$ the minimal bias that one can be guaranteed for an observation of~$f$. Assumptions~\ref{ass:smo} below are made on this function.

At round $t$, the algorithm makes an evaluation of the function of a point $x_t\in \functionspace$ and at a fidelity $z_t\in \Zeta$ (or at a cost $c_t$, see above), as long as $\sum^t_{s=1}\lambda(z_s)\leq\Lambda$. The algorithm observes at round $t$ the value $f_{z_t}(x_t)$ in return. The algorithm must finally output a value $\finalvalue$. We then define the simple regret of a policy $\policysymbol$ for $\regretl$ as 
\[\regretpolicy \triangleq \EE{\max_{x\in \functionspace}\targetf (x) - \targetf(\finalvalue)},\]
where the expectation is taken over the randomness of the algorithm. In the rest of this paper, we only aim to minimize this regret, without any constraint on time or space complexity.

\paragraph{Problem setting remarks}
One of the main aspects our approach is that we \textit{do not assume to have an access to the bias function} $\zeta$. This highlights the fact that our algorithm is fully adaptive, and only needs the cost of each fidelity as an input. Since the bias function is usually unknown in practice, prior works rely on various techniques (e.g., MLE) to guess the values of this function for implementations, but often assume it has a specific form. Given the known results is therefore surprising we get faster rates, and we do it \textit{without relying on any information on the bias function}.

Moreover, we relax the original assumption of \citet{sen2018} that $\Vert \targetf - f_{z} \Vert_{\infty}\leq\zeta(z)$ for $z \in \Zeta$ and use instead $\Vert \targetf - g_z \circ f_{z} \Vert_{\infty}\leq\zeta(z)$ for $z \in \Zeta$.
This lets the $f_z$ approximations be potentially arbitrarily biased with respect to $f$ as long as the ordering in $f$ is approximately kept. Indeed, as $g_z$ are increasing functions,  we have that 
$g_z \circ f_{z} (x_1)  \geq  g_z \circ f_{z} (x_2)$ if and only if 
$ f_{z} (x_1)  \geq  f_{z} (x_2)$ for any $x_1,x_2\in\functionspace$.
 This more general model  for example fits in cases where evaluating at lower fidelities (with higher bias) has a great impact on individual feedback, but a low impact on how each different points compare to each other at the same fidelity level. This is for example the case in neural network training, where evaluating with fewer iterations (lower fidelity) may increase the overall error for every set of hyper-parameters at similar rates.
Note that theoretical results will simultaneously hold under both assumption as long as two conditions are met. First, the behavior of our algorithms is not based directly on the (estimated) value of the function $f_z$  but only on comparisons of these estimates of $f_z$. Second the estimates that are compared are computed from evaluations coming from the same fidelity. 
This is the case of our algorithm \kometo. Indeed, \kometo, similarly to SOO (as noted by \citealp{munos2014book}) or SequOOL, is a \textit{rank-based algorithm}. This means that its behavior is based on the rank of the function evaluations, and not directly on their values. 
On the contrary the behavior of \MFPDOO relies directly on the values in practice when estimating the constant of the parametric model, and would therefore not extend to our general assumption.

Another particularity is that we do not assume that the cost function $\lambda$ is bounded. We assume  quite generally that $\lambda : \Zeta \rightarrow \mathbb{\bar{R_+}}$ instead of restricting ourselves to having  $\lambda : \Zeta \rightarrow [0,1]$ as in \cite{sen2018}. In our scenario, it can happen that some approximations of the function $f$ with low bias are simply too costly for our limited budget.
Working under this larger assumption fits better problems in which we can only access feedback from  imperfect simulators while the real  phenomenon  can not be directly  evaluated in practice.
In such scenario, the MFPDOO~\cite{sen2018} is not usable as it assumes that it directly evaluates the target function $f$ with finite cost during its final cross-validation phase.
Our results can also be extended to cases where fidelity space is discrete, by using a piecewise constant cost function.

Finally instead of minimizing the simple regret, 
the cumulative regret has been also studied in multi-fidelity setting \citep{kandasamy2019multi-fidelity}, rewarding all accurate evaluations of the target function.
However in the present paper we optimize  the simple regret as our initial objective is to find the optimum of the target function. 
 The simple regret is adapted to the objectives of \kometo, e.g., hyper-parameter optimization, where we wish to spend the entire budget on 
\textit{pure exploration}.


\section{Assumptions}
Our algorithm needs two assumptions, one on the target function (which describes its smoothness) and one on the fidelity approximations (which characterizes how well they approximate the function).

\paragraph{Hierarchical partitioning} We use the notion of hierarchical partitioning \citep{munos2011optimistic}. At every depth $h\geq0$, $\functionspace$ (potentially multi-dimensional) is partitioned into $K^h$ different cells $(\partitioninghi)_{0\leq i\leq K^h-1}$. All the cells $(\partitioninghi)_{h,i}$ form a tree, where the root is $P_{0,0}=\functionspace$, and where each cell $\partitioninghi$ has $K$ children, $(\partitioning_{h+1,Ki+l})_{0\leq l\leq K-1}$, which form a partition of their parent cell.

We make an assumption on the target function $f$ and the hierarchical partitioning $\partitioning$, identical to the settings of \cite{sen2018}. This following assumption is much weaker than global Lipschitzness and as explained by \cite{grill2015black-box} is simpler and weaker than assumptions made in previous works \citep{auer2007improved, munos2011optimistic}.

\begin{assumption}[Assumption on the target function]\label{ass:opt}
For one of the global optimum $\optimum$ of $\targetf$, there exists $\nu > 0$ and $\rho\in \left] 0,1 \right[$ such that $\forall h\in \mathbb{N}, \forall x\in\partitioning_{h,i^\star_h}, \targetf(x)\geq\targetf(\optimum)-\nu\rho^h$, where $\partitioning_{h,i^\star_h}$ is the cell of depth $h$ containing $\optimum$
\end{assumption}
We now define a notion of near-optimality dimension that only depends on the hierarchical partitioning $\partitioning$, and not on a metric. 

\begin{definition}
\textbf{Near-optimality dimension}: For any $\nu >0$ and $\rho\in \left] 0,1 \right[$, we say that $d\in\mathbb{R_+}$ is a near optimality dimension of $\targetf$ with respect to the partitioning $\partitioning$ and the smoothness parameters $(\nu,\rho)$ if
\[\exists C>1,  \forall h\in \mathbb{N}, \ncell_h(3\nu\rho^h)\leq C\rho^{-dh},\]
where $\ncell_h(\epsilon)$ is the number of cells $\partitioning_{h,i}$ such that \[\sup_{x\in \partitioning_{h,i}} f(x) \geq \targetf(\optimum)-\epsilon.\]
We then define 
\[
\begin{aligned}
\smoothset\triangleq\ & \left\{ f : X \rightarrow \mathbb{R}\;| f\textrm{ has smoothness}\right.\\
&\textrm{parameters } (\nu,\rho) \textrm{ for } \partitioning \textrm{ and $d$ is a}\\
&\textrm{near-optimality dimension with}\\
&\left. \textrm{associated constant C} \right\}
\end{aligned}
\]
\end{definition}

\paragraph{Note on the near-optimality dimension definition: }\cite{grill2015black-box} define \textit{the} near-optimality dimension as the infimum of the set of d that satisfies this definition (with $\nu$ and $\rho$ fixed). However, they then assume that this infimum also satisfies this definition, which is not necessarily true (the set can be of the form $\mathbb R_{> 0}$ for example). \citet{bartlett2019simple} solve this issue by adding an extra dependence on the constant $C$ to get a closed set (fixed parameters are then $\nu$, $\rho$ and C, instead of just $\nu$ and $\rho$). To avoid this extra dependence, we chose to define $d$ as \textit{a} near-optimality dimension, rather than \textit{the} near optimality dimension.

We can notice that a function with smoothness parameters $(\nu,\rho)$  has necessarily an associated constant
\[C\geq \cmin\triangleq\left(\frac{K}{\rho^{-d}}\right)^{\left\lfloor\frac{\log\;3}{\logrho}\right\rfloor}\] since all
the cells at depth $h_0\triangleq\left\lfloor{\left(\log 3\right)}/{\left(\log 1/\rho\right)}\right\rfloor$ are near-optimal because of \assumf. Indeed, it guarantees that  $\forall x\in X, \targetf(x)\geq \targetf(\optimum) - \nu$, which implies that $\forall x\in X, \targetf(x)\geq \targetf(\optimum) - 3\nu\rho^{h_0}$

We also have that $\dmax\triangleq{\left(\log\;K\right)}/{\left(\log 1/\rho\right)}$ is always a near-optimality dimension of the function, because of the bound $K^h$ on the number of cell of depth $h$.  This emphasizes the fact that the near-optimality dimension of a function is a way to characterize the complexity of optimizing the function, and not an assumption. The case $d=0$ allows for faster rates the best empirical results. As explained by \citet{munos2014book}, the case $d=0$ is the most relevant in practice and covers most of the real-world setups.

We now state three new different assumptions on the rate at which the cost-to-bias function $\Phi$ is decreasing, namely polynomially, exponentially, or by a constant. 


\begin{assumption}[Assumption on the fidelities]\qquad
\label{ass:smo}
\begin{enumerate}[label=(\alph*)]
\vspace{-2mm}
    \item \label{assumsa}
    There exist $A,\alpha>0$ such that $\Phi (c)\leq  A/c^\alpha.$
    \item \label{assumsb} 
    There exist $B,\sigma,\beta >0$ such that $\Phi (c)\leq Be^{\frac{-c^\beta}{\sigma}}.$
\item \label{assumsc} 
There exists $a\geq 1$ such that $\Phi (c)=0$ for all $c\geq a.$
\end{enumerate}
\end{assumption}

\begin{definition} For $\assumsymbol$ being one of the three Assumptions~2 (either \assumsa, \assumsb, \assumsc, with its specific parameters depending on the case), we define
$\assumset$=\{$\fidfamily |$ there exists a function $\zeta$ such that assumption $\assumsymbol$ holds on~$\targetf$ and $\fidfamily$, with $\lambda$ as a cost function and $\zeta$ as a bias function\}.
\end{definition}

The above assumptions describe realistic rates for the cost-to-bias function. \assumsa
generalizes  Assumption~3 of \cite{sen2018} which is equivalent to the case  $\alpha\geq 1$. 
 \assumsb generalizes  Assumption~2 of \cite{sen2018} which  corresponds to the case $\beta=1$.  \assumsc is relevant when a minimal cost to get a perfectly accurate estimation is needed but unknown. It is also useful to link our results (especially the theorem below) to works using single-fidelity optimization, since the settings are then equivalent to deterministic single-fidelity settings.

\section{Lower bound}
We provide the first lower bounds for the assumptions on the fidelities considered.
Theorem \ref{thm:lb} gives, for assumptions \ref{ass:smo}\ref{assumsa}, \ref{ass:smo}\ref{assumsb} and \ref{ass:smo}\ref{assumsc}, bounds on the achievable theoretical performance of an algorithm working under these assumptions.

\begin{theorem}[Lower bounds on simple regret]\label{thm:lb}
Let $\partitioning$ be a partitioning of a space $\functionspace$, $(\nu,\rho)$ some smoothness parameters, $d\in \left[0,\dmax\right]$ a near-optimality dimension with associated constant $C\geq \cmin$ and~$\assumsymbol$ one of the three Assumptions~\ref{ass:smo} with associated parameters.
Then, for any budget $\Lambda$ large enough, for any (deterministic or random) policy $\policysymbol$,
there exist a target function $\targetf\in\smoothset$, a cost function $\lambda$, and fidelity approximations $\fidfamily \in \assumset$ such that: \\
Under \assumsa ($\Phi (c)\leq A/c^\alpha$):
$$\regretpolicy\geq D_1\;\Lambda^{\frac{-1}{d+\frac{1}{\alpha}}}$$
Under \assumsb ($\Phi (c)\leq Be^{\frac{-c^\beta}{\sigma}}$):
$$\regretpolicy\geq \;\begin{cases} e^{-D_2\Lambda^{\frac{\beta}{1+\beta}}}, & \text{when } d=0 \\ D_3\;\Lambda^{\frac{-1}{d}}, & \text{when } d>0 \end{cases}$$
Under \assumsc ($\Phi (c)=0$ for all $c\geq a$):
$$\regretpolicy\geq \;\begin{cases} e^{-D_4\Lambda}, & \text{when } d=0 \\ D_5\;\Lambda^{\frac{-1}{d}}, & \text{when } d>0 \end{cases}$$
where $D_1,D_2,D_3,D_4,D_5>0$ are constants that do not depend on $\Lambda$ and $\pi$.
\end{theorem}

\paragraph{Ideas of the proof} The proof is in the appendix. It is based on the construction of a target function and its approximations, such that the algorithm $\pi$ may not reach a certain depth $h$ and open a near-optimal cell at depth $h$. The construction of the target function is done thanks to a tree, whose leaves are cells of the partitioning, and which reflects which cells are near-optimal for the target function. The approximations are made such that we can lower bound the cost that $\pi$ has to invest to get precise enough information. 

We thus have to construct this tree, which is the tricky part of the proof. This implies choosing near-optimal cells that $\pi$ is unlikely to open. We then get that this depth $h$, which depends on the parameters of the problem and on the budget, may not be reached by $\pi$ with a certain fixed probability. We can use this to lower bound the regret.

\paragraph{Link with the upper bounds} 
In Section~\ref{sec:kometo} we give an algorithm that, without any knowledge on $\nu,\rho,d,C,$ and $\Phi$, achieves these rates with additional constants and logarithmic factors. This means that these lower bounds are close to the optimal rate for policies working with these assumptions, both with and without knowledge of these parameters.

The only previous work using hierarchical partitioning optimization with multi-fidelity model and deterministic feedback worked with narrower assumptions as said above. It obtained, under \assumsa with $\alpha\geq 1$, 
a regret of $\bigO((\frac{\Lambda}{\log \Lambda})^{\frac{-1}{d+1}})$, which is only optimal (ignoring constant and log factors) when $\alpha$=1. Under \assumsb, MPFDOO gets, assuming $\beta=1$ a regret of $\bigO((\frac{\Lambda}{\log \Lambda})^{\frac{-1}{d+\epsilon}})$, for any $\epsilon>0$, provided the budget is large enough (with the threshold having a dependence on $\epsilon$), which does not show that this lower bound was reached.

\assumsc let us extend our results to single-fidelity algorithm with deterministic feedback. A true exponential decay for $d=0$  (and thus optimal up to a constant) was first achieved by DOO (\citealp{munos2011optimistic}, but required the knowledge of the smoothness. SequOOL (\citealp{bartlett2019simple}) then managed to achieve an exponential decay without the knowledge of the smoothness, but with a logarithmic factor in the exponent. We however realized it is possible to get a true exponential decay without the knowledge of the smoothness parameters by changing the number of opened cells at each depth h of SequOOL, to either $\left\lfloor 2\sqrt{n/h}\right\rfloor$ up to depth n, or $\left\lfloor n/\left(h\log({n}/{h})^2\right)\right\rfloor$ up to depth $\left\lfloor n/e^2\right\rfloor$.

\section{Algorithm}\label{sec:kometo}
In this section we propose a new algorithm for multi-fidelity optimization called \kometo. We start with some helpful notation.
\paragraph{Cell evaluations:} Cell evaluations are done through a single representant of each cell $\partitioninghi$, denoted $x_{h,i}$. $T_{h,i,j}$ denote the number of evaluation potentially done for the cell $\partitioninghi$ at fidelity level $j$.

For \kometo, the fidelity level j, with j a non-negative integer, is defined as $z_{e^j}$. At each fidelity level, at most one evaluation can be done for each cell, which means that $T_{h,i,j}$ is equal to either 0 or 1. We hence denote as $f_{h,i,j}$ the result of the potential evaluation, when $T_{h,i,j}=1$. We can notice that, for any j, because of how the cells are opened, \{$\partitioninghi$, $T_{h,i,j}$=1\} is always a tree.

We also slightly modify the usual definition of a cell opening to make it work with our multi-fidelity settings.

\paragraph{Multi-Fidelity Cell Opening:} Opening a cell at fidelity level j means that, for each of its children $\partitioning_{h+1,i}$, the $T_{h+1,i,u}$ for $0\leq u\leq j$ are set to 1. 

This means that the values $f_{h+1,i,u}$, equal to $f_{z_{e^u}}(x_{h+1,i})$, with $x_{h,i}$ the representative element of the cell, can be requested and hence the evaluations can be performed. With this definition, the opening at fidelity level j of a cell can not induce a total cost of more that $\frac{Ke^{j+1}}{e-1}$.

\begin{algorithm}[t]
\caption{KOMETO}
\label{al:kometo}
\begin{algorithmic}[1]
		\STATE
			  \textbf{Parameters:}  $\fidfamily$, $\partitioning$, $\Lambda$, $\lambda$\\[.05cm]
				\STATE \textbf{Init:}\\
			  
			 $\tilde{\Lambda}\gets \left\lfloor\frac{(e-1)\Lambda}{2Ke\left(\log\ \Lambda+1\right)^2}\right\rfloor\CommaBin$ $\jmax\gets \left\lfloor\log\ \tilde{\Lambda}\right\rfloor$.\\
			 Open with budget $\hmax$ the cell $P_{0,0}$.\\[.45cm]
			\textbf{For}\;$h=1$ to $\lfloor\hmax\rfloor$  \textit{\textbf{{$\hfill\blacktriangleleft$  Exploration $\blacktriangleright$}}}
			
			~~        \textbf{For} $m=1$  to $\left\lfloor\hmax/h\right\rfloor$\vspace{.05cm}
			
			~~~~\, $j\gets\lfloor\log\ \frac{\tilde{\Lambda}}{hm}\rfloor$ \vspace{.05cm}
			~~~~\, Open at fidelity level $j$ the non-opened\\       \phantom{bbbbbb}
			$\;\;\;\;\;\;\;$cell $P_{h,i}$ with the highest value $f_{h,i,j}$,
			\\  \phantom{bbbbbb}   $\;\;\;\;\;\;\;$given that  $T_{h,i,j}= 1$ \\[.45cm]
			\textbf{For} $j= 0$ to $\jmax$ \textit{\textbf{{\hfill $\blacktriangleleft$  Cross-validation $\blacktriangleright$}}} \vspace{0.05cm} 
			 
			$\quad$ \textbf{Evaluate} at cost $\hmax$ the 
			candidates\\
			\phantom{aaaaaa} $x^c_j \gets \argmax\limits_{\left(h,i\right)\in \tree,\,T_{h,i,j}=1} f_{h,i,j}$. 
			
			\textbf{Output} $\finalvalue \gets \argmax\limits_{\{j\in[0:\jmax]\}} {f_{z_{\hmax}}}\left(x^c_j\right)$
\end{algorithmic}
\end{algorithm}

\paragraph{\kometo explanations:}
\kometo is detailed in Algorithm~\ref{al:kometo}.
The algorithm presented is inspired by StroquOOL  (\cite{bartlett2019simple}). Its main feature is that, using Zipf sampling (which means, opening up to $\hmax$ cells at h=1, up to $\hmax/2$ cells at h=2 and so on) 
it manages to reach the optimal rate up to logarithmic factor without the knowledge of the smoothness. This is done, in the exploration part, by opening a decreasing number of cell at each depth, and at a given depth, gradually decreasing the fidelity at which cells are opened. The intuition behind this idea is that, for each depth h, and each $0\leq j_h\leq \jmax$, the number of cell opened at fidelity level $j_h$ or higher will decrease with $j_h$. If this $j_h$ is too low, the precision might also be too low for the choices to be relevant, but if $j_h$ is too high, not enough cells will be opened. Cross-validation is then used by the algorithm in order to choose the best cell regardless of depth and fidelity level. It ensures that the choice of a particular $j_h$ is not needed.
\\

\paragraph{Budget optimization:}  With a given budget $\Lambda$, we can actually initialize the $\hmax$ constant with a way higher value than $\left\lfloor\frac{(e-1)\Lambda}{2Ke\overline{\log}\ (\Lambda+1)^2}\right\rfloor\CommaBin$ for multiple reasons:
\begin{itemize}
\itemsep -1em
    \item 
 The actual cost used for a cell opening is rounded down to $e^{\lfloor\log c\rfloor}$.\\
    \item   The total budget mentioned for a cell opening assumes that all the evaluation at different fidelities will be requested for the children, which is not the case.\\
    \item   The number of opened cells at each depth is bounded by $K^h$\\
    \item   For some partitioning, it is possible to use the parent evaluations for one of its children.
\end{itemize}
Since the budget used can be predicted using only $\hmax$ and the partitioning, and increase with $\hmax$, it is possible to quickly calculate the optimal initial value  of $\hmax$ using dichotomy. However, these previous optimization can only increase $\hmax$ by a multiplicative constant.
Even if the budget needs to be set in advance for this algorithm, since we optimize the simple regret, we can obtain any-time guarantees which only differ by a multiplicative constant using the doubling trick. 
\section{Theoretical guarantees}
We first state a simple proposition which asserts, with the initial value of $\hmax$, the budget condition is respected.

\setcounter{proposition}{1}
\begin{proposition}[Budget use]
\label{prop:bud}
The budget used by \kometo does not exceed $\Lambda$.
\end{proposition}
Our upper bounds use the Lambert function, evaluated at positive real values. This function is defined as the inverse of the function $f(z)=ze^z$.
With the first two terms of its asymptotic expansion, we get, when $z$ goes to infinity, that $W(x)=\log x  - \log\log x + o(1)$.

We now state the main results of our analysis, using the same context as Theorem 1 on lower bounds. The proof is given in appendix:
\setcounter{theorem}{2}
\begin{theorem}[Upper bounds on the simple regret]
\label{thm:ub}
Let $\partitioning$  be a partitioning of a space $\functionspace$, $(\nu,\rho)$ some smoothness parameters, $d\in \left[0,\dmax\right]$ a near-optimality dimension with associated constant $C\geq \cmin$ and $\assumsymbol$  being one of the three Assumption 2 with its associated parameters.\\
Then, for any budget $\Lambda\geq 1$, target function $\targetf\in\smoothset$, cost function $\lambda$, and fidelity approximations $\fidfamily \in \assumset$ provided to \kometo,\\\\
Under \assumsa ($\Phi (c)\leq \frac{A}{c^\alpha}$):
We first define two values, then state the regret\\

\tabulinesep=1mm
\begin{tabu}{ |c|  }
 \hline
 \multicolumn{1}{|c|}{Value of $h_1$} \\
 \hline\hline
 $\frac{1}{\dalpharho}W\left(\frac{\hmax \nu^\invalpha\dalpharho}{4CeA^\invalpha}\right)$\\
 \hline
\end{tabu}

\tabulinesep=1mm
\begin{tabu}{ |c||c|  }
 \hline
 \multicolumn{2}{|c|}{Value of $h_2$} \\
 \hline\hline
 When $d=0$& $\frac{\hmax}{4C}$\\
 \hline
 When $d>0$& $\frac{1}{d\logrho}W\left(\frac{\hmax d\logrho}{4C}\right)$\\
 \hline
\end{tabu}

\tabulinesep=1mm
\begin{tabu}{ |c||c|  }
 \hline
 \multicolumn{2}{|c|}{Regret} \\
 \hline\hline
 \textbf{High budget} ($\nu\rho^{h_1}\leq e^\alpha A$)& $\regretl \leq \frac{3\nu}{\rho}\rho^{h_1} +2\frac{A}{\hmax^\alpha}$\\
 \hline
 \textbf{Low budget} ($\nu\rho^{h_1} > e^\alpha A$)& $\regretl \leq \frac{3\nu}{\rho}\rho^{h_2} +2\frac{A}{\hmax^\alpha}$\\
 \hline
\end{tabu}

Under \assumsb ($\Phi (c)\leq Be^{\frac{-c^\beta}{\sigma}}$):
We also define $\abnu=\max\left(\frac{1}{2\sigma},\log\left(\frac{B}{\nu}\right)\right)$

\tabulinesep=1mm
\begin{tabu}{ |c||c|  }
 \hline
 \multicolumn{2}{|c|}{Value of $h_1$} \\
 \hline\hline
 When $d=0$& $\left(\frac{\hmax}{4Che}\right)^\frac{\beta}{\beta +1} \left(\frac{1}{2\sigma \logrho}\right)^\frac{1}{\beta +1}$\\
 \hline
 When $d>0$& $\frac{\beta +1}{\beta d\logrho}W(\frac{\beta}{\beta +1}d\logrho\left(\frac{\hmax}{4Che}\right)^\frac{\beta}{\beta +1} $\\
 &$
 \left(\frac{1}{2\sigma \logrho}\right)^\frac{1}{\beta +1})$\\
 \hline
\end{tabu}

\tabulinesep=1mm
\begin{tabu}{ |c||c|  }
 \hline
 \multicolumn{2}{|c|}{Value of $h_2$} \\
 \hline\hline
 When $d=0$& $\frac{\hmax}{4Ce\left(2\sigma \abnu\right)^\invbeta}$\\
 \hline
 When $d>0$& $\frac{1}{d\logrho}W\left(\frac{\hmax d\logrho}{4Ce\left(2\sigma \abnu\right)^\invbeta}\right)$\\
 \hline
\end{tabu}

\tabulinesep=1mm
\begin{tabu}{ |c||c|  }
 \hline
 \multicolumn{2}{|c|}{Regret} \\
 \hline\hline
 \textbf{High budget} ($h_1 \geq \frac{\abnu}{\logrho}$)& $\regretl \leq \frac{3\nu}{\rho}\rho^{h_1} +2Be^{\frac{-\hmax^\beta}{\sigma}}$\\
 \hline
 \textbf{Low budget} ($h_1< \frac{\abnu}{\logrho}$)& $\regretl \leq \frac{3\nu}{\rho}\rho^{h_2} +2Be^{\frac{-\hmax^\beta}{\sigma}}$\\
 \hline
\end{tabu}

Under \assumsc ($\Phi (c)=0$ for all $c\geq a$):

\tabulinesep=1mm
\begin{tabu}{ |c||c|  }
 \hline
 \multicolumn{2}{|c|}{Value of $h$} \\
 \hline\hline
 When $d=0$& $\frac{\hmax}{4Cae}$\\
 \hline
 When $d>0$& $\frac{1}{d\logrho}W\left(\frac{\hmax d\logrho}{4Cae}\right)$\\
 \hline
\end{tabu}
\tabulinesep=1mm
\begin{tabu}{ |c|  }
 \hline
 {Regret} \\
 \hline\hline
 $\regretl \leq \frac{\nu}{\rho}\rho^{h}$\\
 \hline
\end{tabu}
\end{theorem}
\setcounter{corollary}{3}
\begin{corollary}[Regret decreasing rates]
\label{co:ub}
Following Theorem 3
(the exact upper bounds used for the rates are given in appendix): 

\tabulinesep=1mm
\begin{tabu}{ |c||c|c|  }
 \hline
 \assumsa&High budget&Low budget\\
 \hline\hline
 When $d=0$& \multicolumn{2}{|c|}{$\bigOtilde(\Lambda^{-\alpha})$}\\
 \hline
 When $d>0$& $\bigOtilde(\Lambda^\frac{-1}{d+1/\alpha})$&$\bigOtilde(\Lambda^\frac{-1}{d}+\Lambda^{-\alpha})$\\
 \hline
\end{tabu}

\tabulinesep=1mm
\begin{tabu}{ |c||c|c|  }
 \hline
 \assumsb&High budget&Low budget\\
 \hline\hline
 When $d=0$& $e^{\bigOtilde(-\Lambda^\frac{\beta}{1+\beta})}$&$e^{\bigOtilde(-\Lambda^{\beta})}+e^{\bigOtilde(-\Lambda)}$\\
 \hline
 When $d>0$& \multicolumn{2}{|c|}{$\bigOtilde(\Lambda^\frac{-1}{d})$}\\
 \hline
\end{tabu}

\tabulinesep=1mm
\begin{tabu}{ |c||c|  }
 \hline
 \multicolumn{2}{|c|}{\assumsc}\\
 \hline\hline
 When $d=0$& $e^{\bigOtilde(-\Lambda)}$\\
 \hline
 When $d>0$&$\bigOtilde(\Lambda^{\frac{-1}{d}})$\\
 \hline
\end{tabu}

\end{corollary}

As explained in the next paragraph, in practice and for asymptotic comparisons only the results for high budget settings are relevant.

We can notice that the rates of decreasing are better until the threshold for high budget. This is because, until the threshold, the algorithm does not have to focus on increasing the fidelity cost to improve the result, since the improvements in the regret it can make by exploring more cells is vastly superior to the improvements it can make with more precise analysis (which involves more precise evaluations: a higher fidelity). This explains why the rates are close to the one obtained on single-fidelity optimisation (or, similarly, on \assumsc, which materializes this case). However, the low budget case actually requires a very low budget (or very accurate fidelities) so these rates are not really relevant in practice. This dichotomy was similarly noticed, by \cite{bartlett2019simple} for the StroquOOL algorithm, in a stochastic case: using only one evaluation was enough as long as the noise did not exceed the potential regret that could be obtained.

\begin{figure*}
\begin{center}\vspace{-.5cm}
\includegraphics[width=0.35\textwidth]{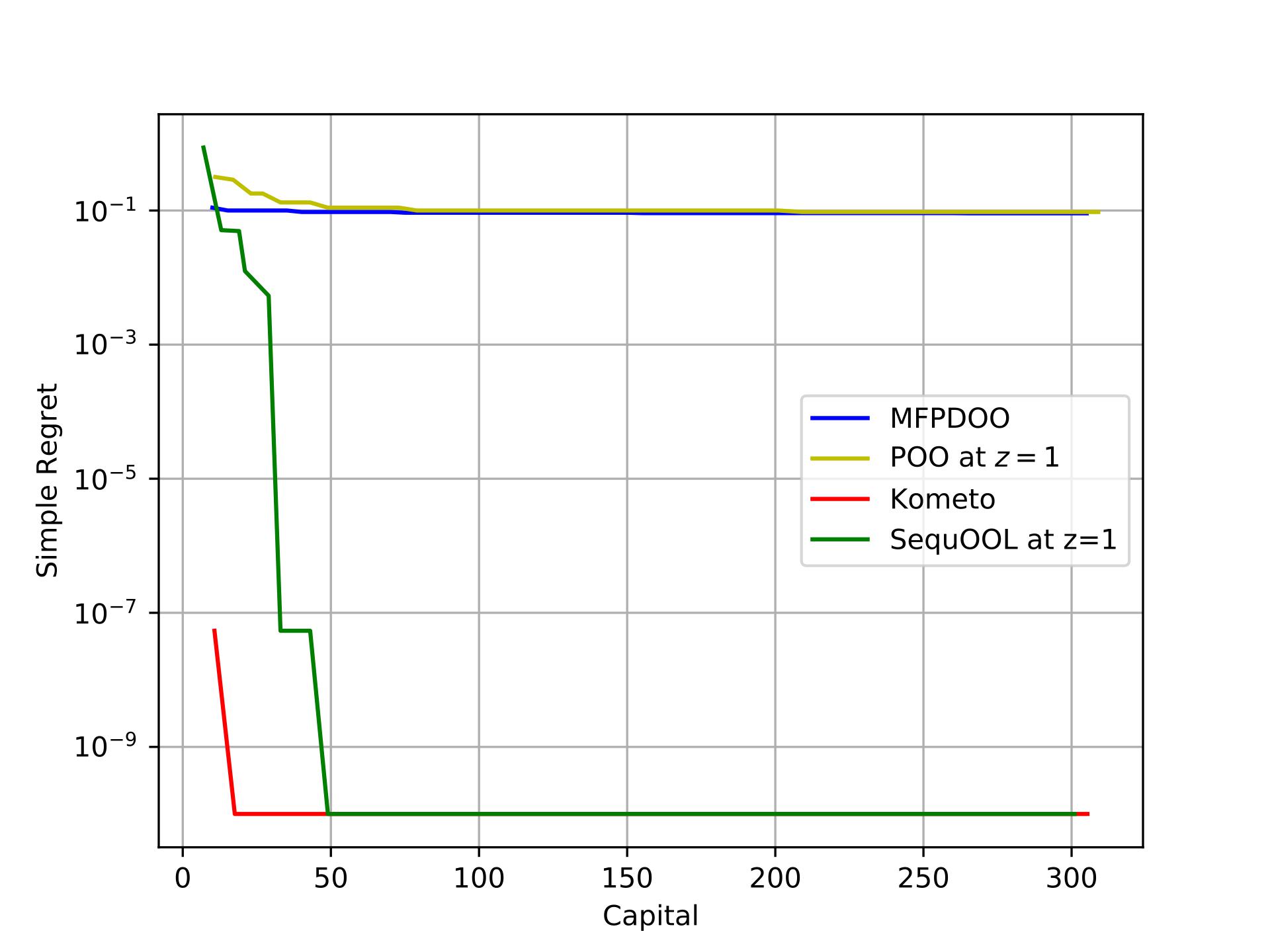}\hspace{-.7cm}
\includegraphics[width=0.35\textwidth]{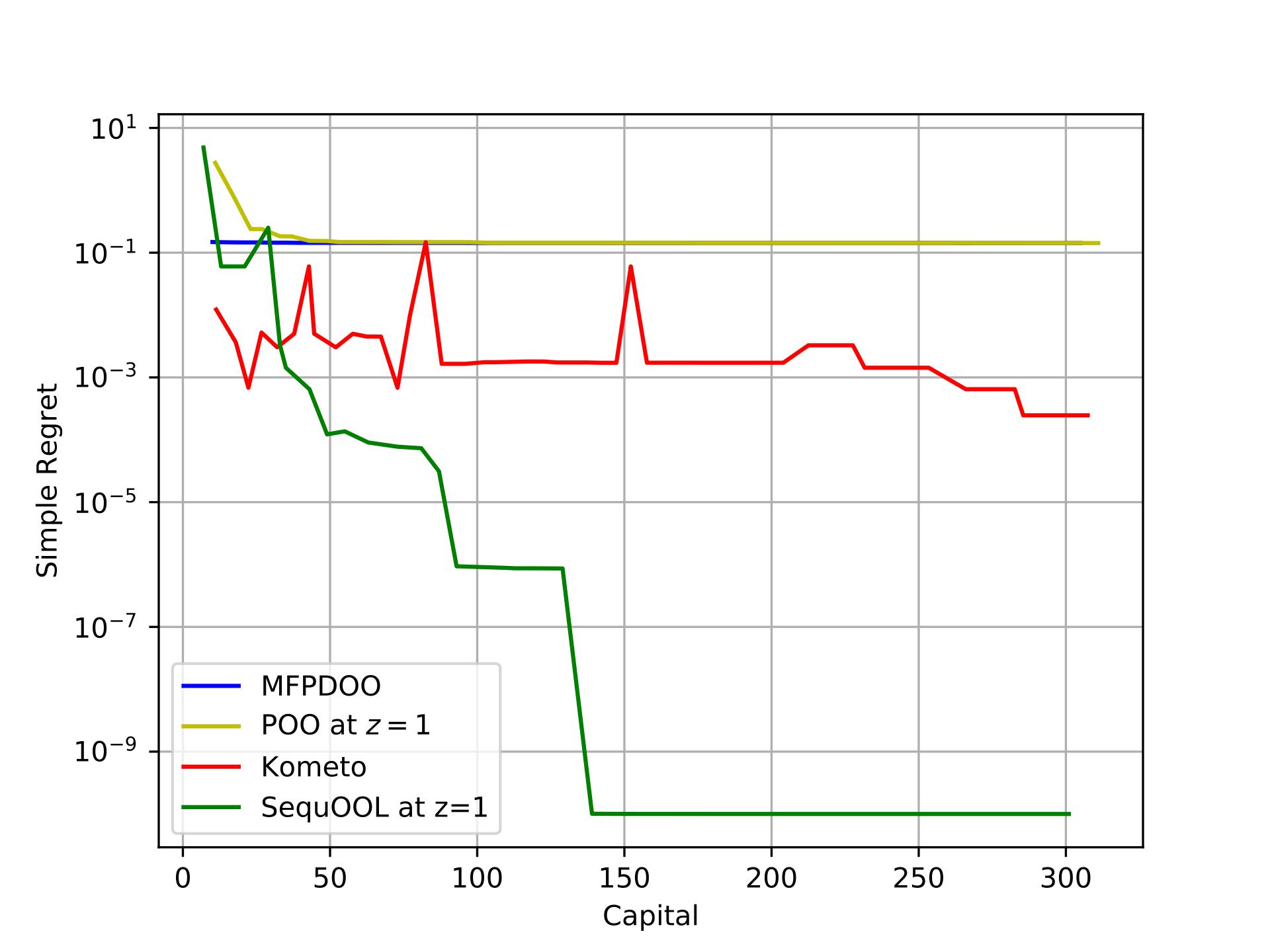}\hspace{-.7cm}
\includegraphics[width=0.35\textwidth]{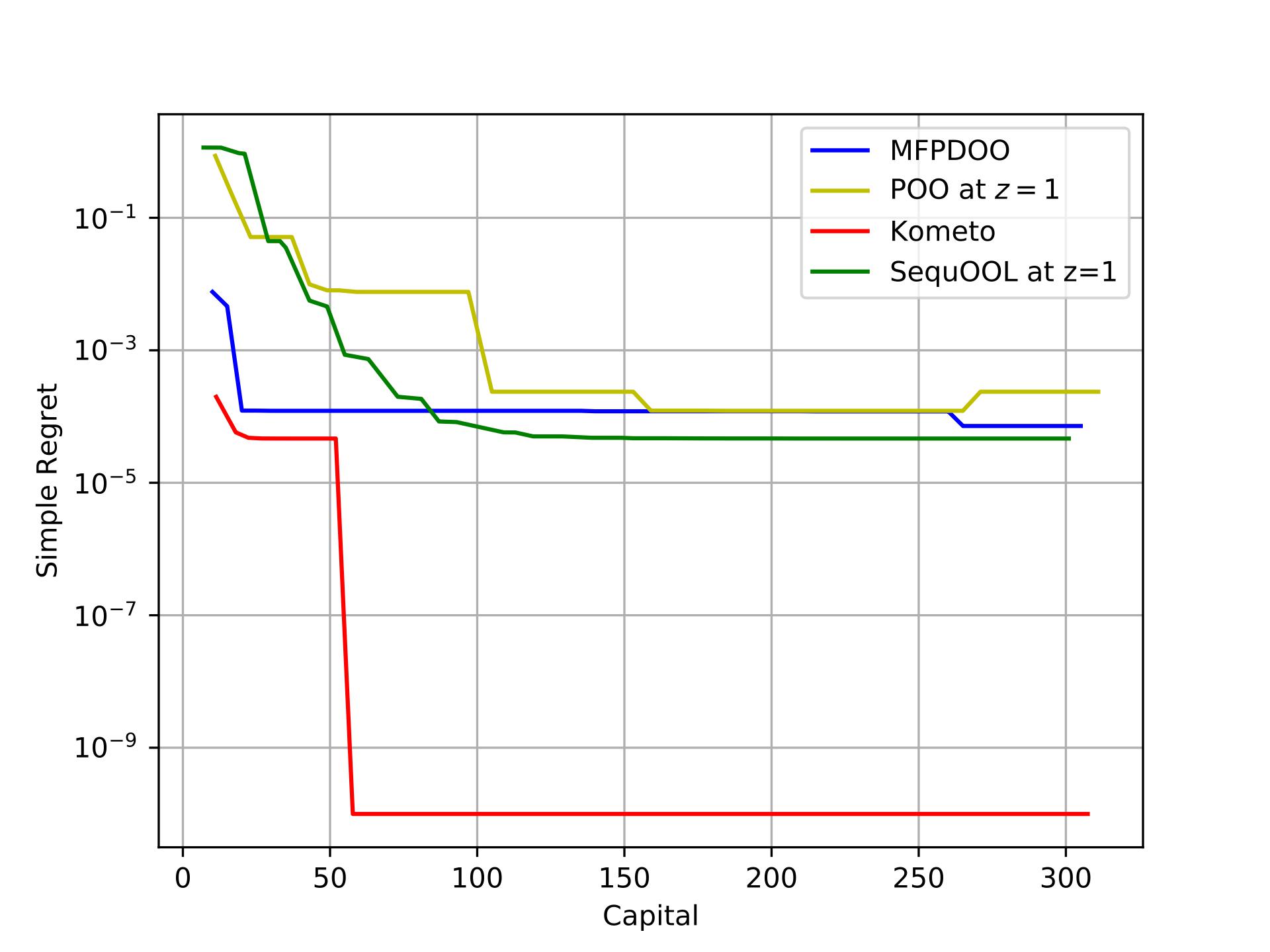}
\includegraphics[width=0.35\textwidth]{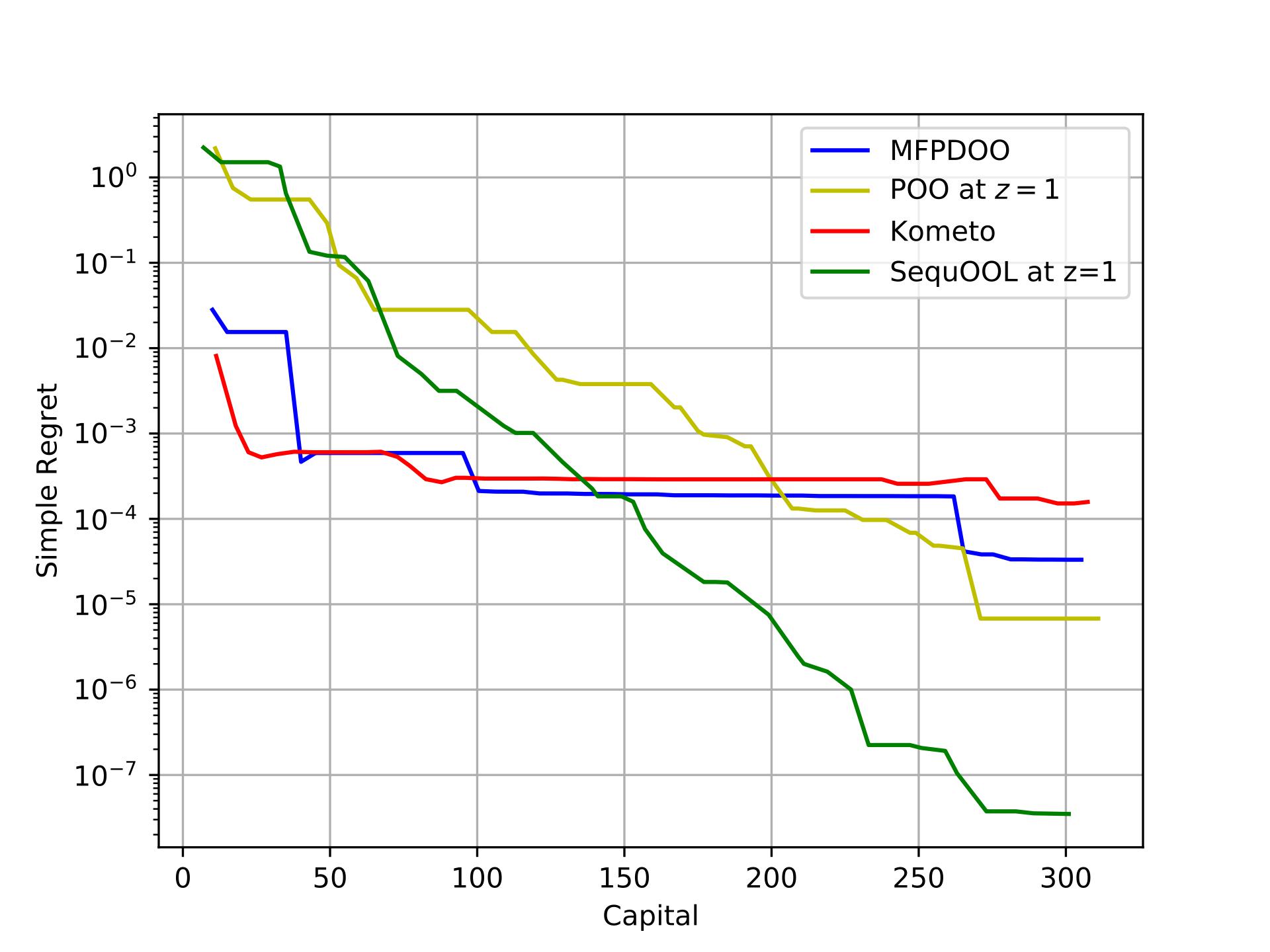}\hspace{-.7cm}
\includegraphics[width=0.35\textwidth]{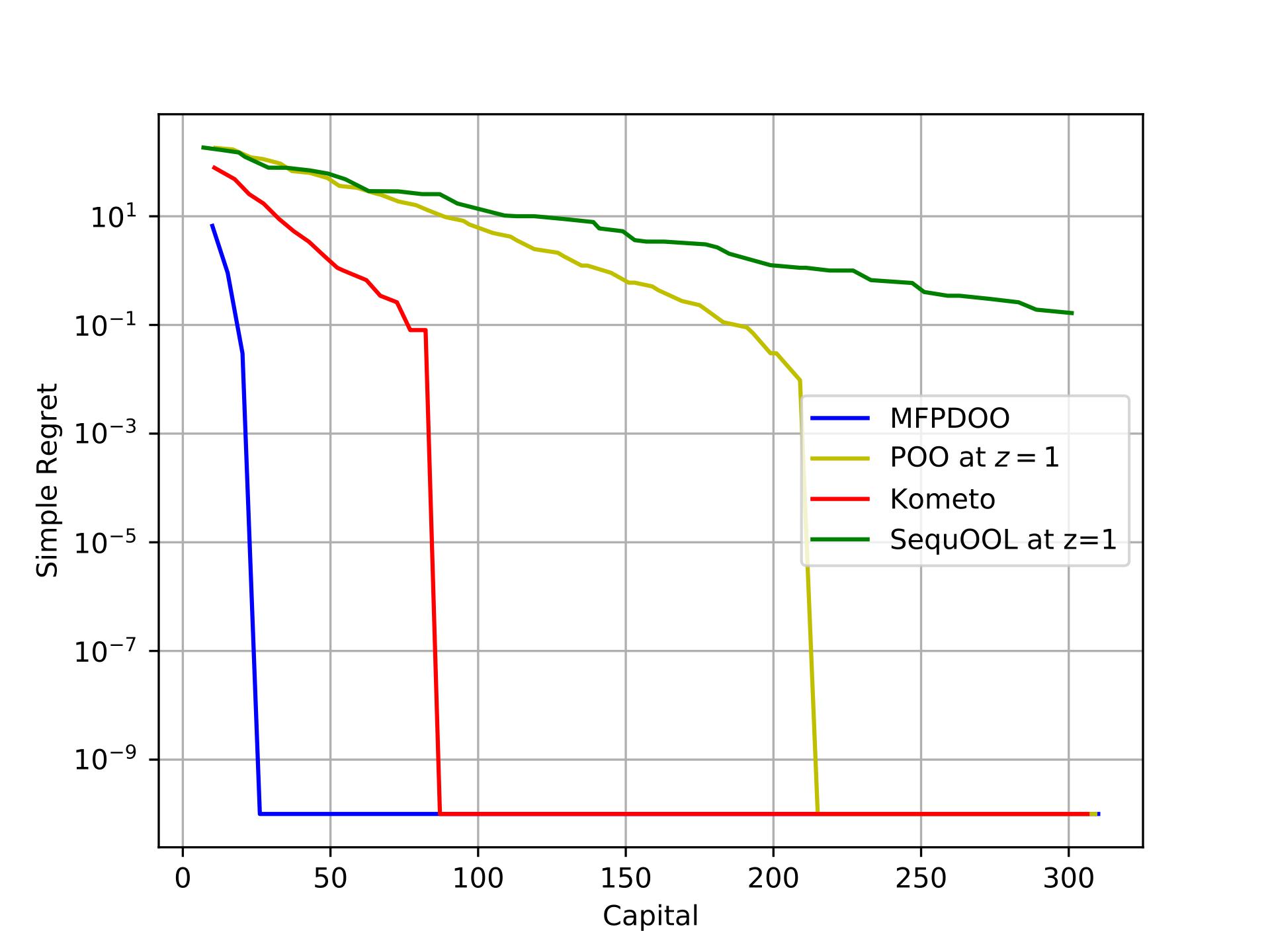}\hspace{-.7cm}
\includegraphics[width=0.35\textwidth]{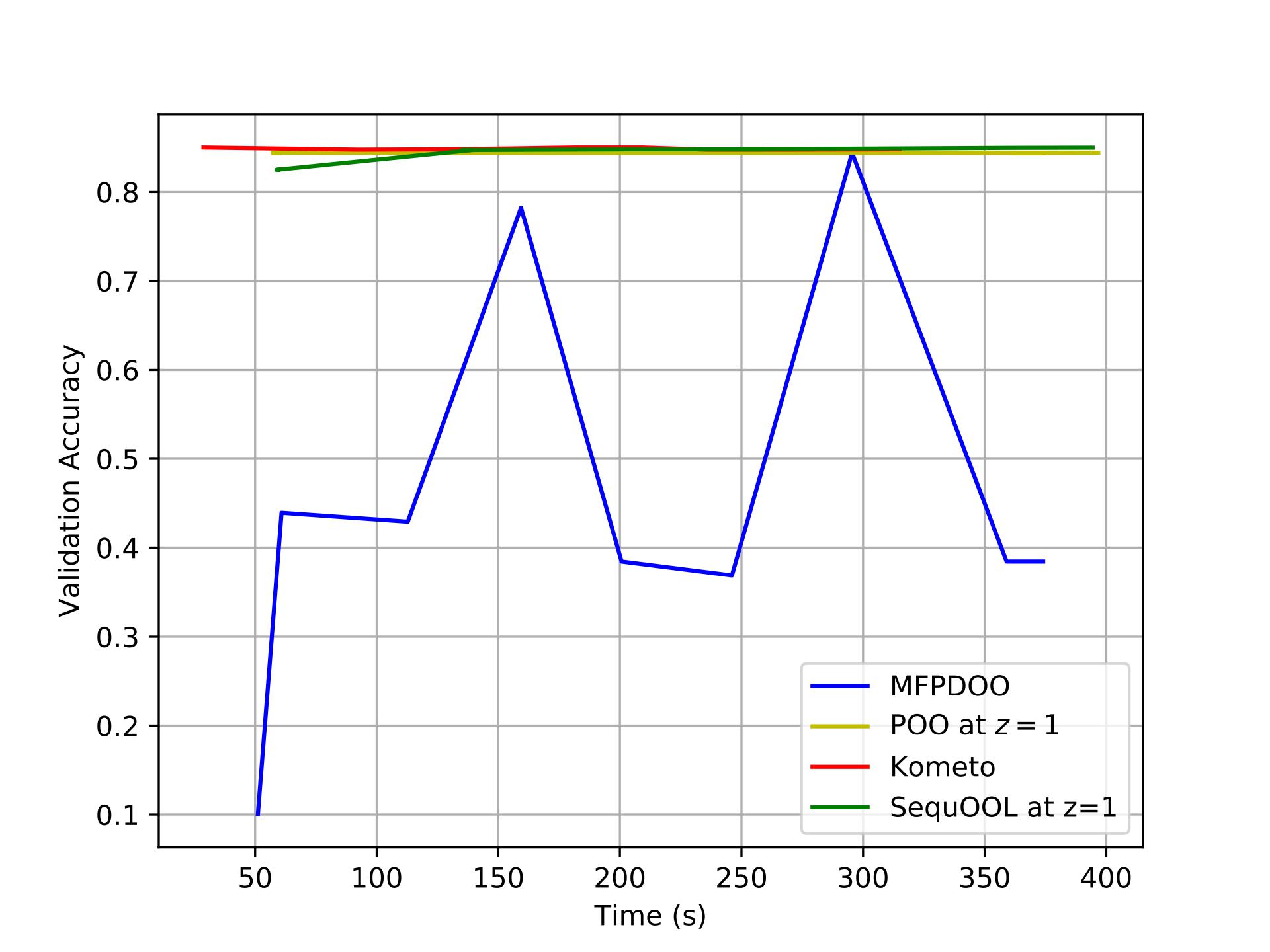}
\caption{\textbf{(a) top-left}: Curin 2-dimensions,
\textbf{(b) top-center:} Branin 2-dimensions
\textbf{(c) top-right:} Hartman3d 3-dimensions
\textbf{(d) bottom-left:} Hartman6d 6-dimensions
\textbf{(e) bottom-center:} Borehole 8-dimensions
\textbf{(f) bottom-right:} SVM 2-dimensions.
Experiments are composed of five synthetic experiments, from (a) to (e), and one real-world, (f).
The multi-fidelity algorithms can use all fidelities, while the non multi-fidelity algorithms only request at fidelity $z=1.$
The x-axis gives the budget effectively used by the algorithm, to reflect algorithm which exceed the attributed budget. The y-axis denotes the regret for the synthetic experiments (the lower the better), and the accuracy for the SVM experiment (the higher the better).
For readability, the graph only plot the regret down to $10^{-10}$.\vspace{-.63cm}
}\end{center}\vspace{-.1cm}
\end{figure*}

\section{Empirical results}
We chose to do the same synthetic and practical deterministic experiments as those done in \cite{sen2018}, and used their code for fair comparisons. The algorithms to which \kometo is compared are MFPDOO (\cite{sen2018}), POO (\cite{grill2015black-box}) and SequOOL (\cite{bartlett2019simple}). We directly used \kometo without any tweaking. This shows \kometo's adaptability, which only needed the cost function and the space $\functionspace$ in order to work.

\paragraph{Experiments explanation}
Five of them are synthetic deterministic experiments of different, but always low, dimensions. The budget is expressed in terms of the number of multiple of the highest fidelity cost $\lambda(1)$. Note that these experiments may easily be unfair toward non multi-fidelity algorithms, because the results of the multi-fidelity algorithms heavily depends on how useful the low fidelities are, which is arbitrary on synthetic experiments. Therefore, since non multi-fidelity algorithms have no access to low fidelities and thus have less information, synthetic experiments should not be used to directly compare the efficiency of a multi-fidelity and a non multi-fidelity algorithm.

The last experiment aims to measure the efficiency of the algorithms in practical settings. It involves tuning two hyperparameters for text classification, with the number of samples used to obtain 5-fold cross-validation accuracy determined by the fidelity. The budget is, for this experiment, determined by the time used by the algorithm to return its result, reflecting simultaneously the 
actual time used for the algorithm execution and the cost of computing the accuracies.

Details about the experiments, along with comparisons to other multi-fidelity algorithms, can be found in \cite{sen2018}.

\paragraph{Experiments analysis}
We can notice that \kometo largely outperforms MFPDOO on three of the synthetic experiments (Branin, Curin and Hartman3d) and on the practical experiment. It however gets beaten by MFPDOO on Borehole and Hartman6d by a relatively small margin. For the Hartman3d and Curin experiments, the better results of \kometo could be explained with its rank-based property, low fidelities may give highly accurate information on the way close points compare each other on the target function.

Interestingly, SequOOL outperforms \kometo on the Branin and Hartman6d experiments. This happens because, for these experiments, a lot of high-fidelity evaluations are needed to minimize the regret. Since \kometo keeps an important portion of its budget for low-fidelity evaluations, it is late compared to SequOOL which only does high-fidelity evaluations. This is materialized in the theoretical guarantees by the fact that \kometo has additional logarithmic factors compared to SequOOL under \assumsc.

\section{Discussion}

\paragraph{Possible stochastic settings} Our algorithm works in deterministic settings. However, our hypothesis of a bounded bias can be replaced with an hypothesis of a noise (potentially biased), with the same bounds. Our algorithm can therefore work in stochastic settings, the guarantees being given instead at high probability with a cost-to-bias function changed accordingly.

However, in cases where the noise does not naturally decrease to $0$ at higher fidelities, the $\Phi$ function will not decrease to $0$ either although required by \ref{ass:smo}. This issue can be resolved by gradually increasing the number of evaluations at higher fidelities, to get a $\Phi$ function that would converge to $0$. Indeed using concentration inequalities, we could then have Assumption \ref{ass:smo} true with high probability, which could bound the regret.

\paragraph{Cumulative regret in adaptive multi-fidelity optimization}
\cite{locatelli2018adaptivity} states that the minimax optimal cumulative regret with the knowledge of the smoothness cannot be attained by single-fidelity algorithms without the knowledge of the smoothness of the function. We wonder if this result remains true in multi-fidelity settings using adapted cumulative regret definitions.


\clearpage
\bibliography{references,bandits,come}
\bibliographystyle{apalike}
\newpage
\onecolumn
\appendix

\section{Proof of Theorem~\ref{thm:lb}}

In the following proof, we set $\lambda$ a bijective cost function, and $\zeta(z)=\Phi(x_z)$, where $\Phi$ is the upper bound of the cost-to-bias function given by the assumptions. 

We first give a way to construct target and fidelity function for the lower bounds.
Let $\trunctreea$ a truncated tree of $\partitioning$ (more precisely, a tree with the same root as $\partitioning$, and included in $\partitioning$). $\trunctreea$ nodes are therefore sub spaces of $\functionspace$.\\\\
We then define:
\begin{itemize}
    \item For all $h\geq 0$, $\trunctreeh$, the union of all the sub spaces associated to nodes of $\trunctreea$ of depth $h$. We can notice that $(\trunctreeh)_{h\geq0}$ is a decreasing sequence for the inclusion, with $T^0_a=\partitioning_{0,0}=\functionspace$.\\\\
   \item  $\targettrunctree$ by $\targettrunctree(x)=\sup_{h\geq0} \{-\nu\rho^h | x\in \trunctree_a^h\}$, the target function\\\\
   \item  $(\fidtrunctree)_{z\in Z}$ by $\fidtrunctree(x)=min\{f^{\trunctree_a}(x), -\zeta(z)\}$, the fidelities
\end{itemize}

We now state Lemma 5, related to the constructions above. It shows that the previous target and fidelity functions can be used as counter-examples in order to show the lower bounds.

\setcounter{lemma}{4}
\begin{lemma}\label{lem:fun}Let $\trunctree_a$ be a truncated tree of $\partitioning$ such that, for every depth $h\geq 0$, $\trunctreea$ has between $1$ and $\left\lfloor \rho^{-dh}\right\rfloor$ cells of depth h.\\

Then:\\
- $f^{\trunctreea} \in \smoothset$\\
- $(\fidtrunctree)_{z\in Z}\in F(\assumsymbol,\targettrunctree,\lambda)$
\end{lemma}

\begin{proof}We start by showing that $f^{\trunctreea} \in \smoothset$.\\
As $\trunctreea$ has an infinite depth, there exists a sequence $\left(i_h\right)_h\in\mathbb{N}$ of indexes such that $\forall h\in\mathbb{N}$, $\{\partitioning_{h,i_h}\}\in\trunctreea$. By taking $\left(x_h\right)_{h\in\mathbb{N}}$, such that for all $h\in\mathbb{N}$ $x_h\in\partitioning_{h,i_h}$, we have that $\sup \limits_{h\in\mathbb{N}}\targettrunctree\left(x_h\right)=0$, which means $\sup \limits_{x\in X}\targettrunctree\left(x\right)=0$\\
Furthermore, we have that, for all $h\in\mathbb{N}$, for all $x\in\partitioning_{h,i_h}$, $\targettrunctree(x)\geq -\nu\rho^h$ by definition of $\targettrunctree$, which is equivalent to $\targettrunctree(x)\geq\sup\limits_{x'\in X}\targettrunctree\left(x'\right) -\nu\rho^h$. We therefore have that \assumf is true for $\targettrunctree$ for the smoothness parameters $\left(\nu,\rho\right)$.\\\\
We now show that $d$ is a near-optimality dimension of $\targettrunctree$ for these smoothness parameters. Let $h_o=\left\lfloor\frac{\log 3}{\logrho}\right\rfloor$ (we defined $\cmin$ with $\cmin=\left(\frac{K}{\rho^{-d}}\right)^{h_o}$). Let $h\in\mathbb{N}$, let's show that the related assumption is true at depth $h$:\\
If $h\leq h_o$, then the number of near-optimal cells at depth $h$ is simply bounded by the number of cells of depth $h$, $K^h$. As \[K^h=\frac{K^{h_o}}{K^{h_o-h}}\leq\frac{K^{h_o}}{\rho^{-d(h_o-h)}}=\left(\frac{K}{\rho^{-d}}\right)^{h_o}\rho^{-dh}\leq \cmin\;\rho^{-dh},\] the hypothesis is true at depth $h$.\\
If $h\geq h_o$, then thanks to hypothesis (2), there are at most $\left\lfloor\rho^{-d(h-h_o)}\right\rfloor$ cells of height $h-h_o$ such that $\sup\limits_{x\in\partitioning_{h-h_o,i}}\targettrunctree\left(x\right)>-\nu\rho^{h-h_o-1}$.
By taking for each of these cells its $K^{h_o}$ grandchildren, we get that there are at most $\left\lfloor K^{h_o}\rho^{-d(h-h_o)}\right\rfloor$ cells of depth $h$ such that $\sup\limits_{x\in\partitioning_{h,i}}\targettrunctree\left(x\right)>-\nu\rho^{h-h_o-1}$.\\
As \[\left\lfloor K^{h_o}\rho^{-d(h-h_o)}\right\rfloor\leq \cmin\;\rho^{-dh}\] and \[-\nu\rho^{h-h_o-1}=-\nu\rho^{h-\left\lfloor\frac{\log 3}{\logrho}+1\right\rfloor}\geq -\nu\rho^{h-\frac{\log 3}{\logrho}}=-3\nu\rho^h=\sup\limits_{x'\in X}\targettrunctree\left(x'\right) -3\nu\rho^h,\] there are at most $\cmin\rho^{-dh}$ near-optimal cells at depth $h$, which concludes. \\\\

We finally want to show that $(\fidtrunctree)_{z\in Z}\in F(\assumsymbol,\targettrunctree,\lambda)$. As we defined $\zeta(z)=\Phi(x_z)$, where $\Phi$ is the upper bound given by \assumf, we just need to show, by taking $g_z=Id$ for all $z\in Z$, that $\Vert \targettrunctree - \fidtrunctree\Vert_{\infty}\leq \zeta(z)$ for all $z\in Z$.\\
Taking $x\in X$, we either have \\
\alinea - $\targettrunctree(x)\leq -\zeta(z)$, and thus $\fidtrunctree(x)=\targettrunctree(x)$, or\\
\alinea - $\targettrunctree(x)> -\zeta(z)$ which implies $|\targettrunctree(x) - \fidtrunctree(x)|=\targettrunctree(x) + \zeta(z)\leq \zeta(z)$ as the function $\targettrunctree$ is non-positive by definition. \\The distinction between these two cases gives us that the bound on the infinity norm is true, and thus concludes.
\end{proof}
\setcounter{lemma}{5}
\begin{lemma}
\label{lem:lem6}
By choosing appropriate target functions and fidelity approximations, we can get two lower bounds of the regret with $\pi$:\\
\textbf{Lower bound a:} $\regretl\geq \sup\;\left\{ r\in]0,\frac{1}{2}\nu\rho]\;|\;\Lambda \leq \frac{1}{K}\left(\frac{2r}{\nu\rho}\right)^{-d}\:\inf\Phi^{-1}\left([0,\frac{2r}{\rho}]\right)\right\}$\\
\textbf{Lower bound b:} $\regretl\geq \sup\;\left\{ r\in]0,\frac{1}{4}\nu\rho^6]\;|\;\Lambda \leq \left(\frac{\log\frac{\nu}{4r}}{4\logrho}-2\right)\:\inf\Phi^{-1}\left(\left[0,\frac{\nu}{\rho}\left(\frac{4r}{\nu}\right)^\frac{1}{4}\right]\right)\right\}$
\end{lemma}

\begin{proof}
We will denote by $(\pointfidl)_{l\in \mathbb{N}}$ the family of random variables equal to the successive points requested by $\pi$ which has been given the fidelity approximations $\left(\fidtrunctree\right)$ (the values are null when $l$ exceeds the number of evaluations). We also denote, for a given depth $h$, by $\setfidh$ the (also random) set of indexes $l$ such that the fidelity z associated to the request of the point $\pointfidl$ respects $\zeta(z)\leq \nu\rho^h$.\\
The key to the following proof is that, for any depth $h\geq0$, for any tree $\trunctreea$, if $\trunctreeabis$ is the tree $\trunctreea$ whom nodes of depth $h'\geq h+1$ have been cut, then for any $l\geq 0$, $\pointfidl\; |\; (\forall m\in \setfidh, m<l \implies X^{\trunctreea}_m\not\in \trunctreea^{h+1})$ has the same distribution as $\pointfidlbis\; |\; (\forall m\in \setfidhbis, m<l \implies X^{\trunctreeabis}_m \not\in \trunctreea^{h+1})$. This is because, in this instance, all the previous evaluations were either at a fidelity z too low and were equal to $-\zeta(z)\leq -\nu\rho^h$, or were evaluated outside of the sub space of a node of $\trunctreea$ of depth h+1. In both cases, the returned values were unaffected by the cutting of the nodes of depth $h'\geq h+1$ of $\trunctreeabis$, which implies that the behavior of the algorithm will remain the same for the request of the point $X^{\trunctreeabis}_l$.\\

We thus define the idea of opening a cell C as, when C is of depth $h$, having a $l\in\setfidh$ such that $\pointfidl\in C$.\\ This means requesting a point of this cell with a high enough fidelity to differentiate the cells of $\trunctreea^{h+1}$ from the rest of the cells.\\

To simplify the proof, we will assume that the output of the algorithm is done with a last free evaluation at the highest fidelity. This evaluation will thus belong to $\setfidh$ for any $h\geq 0$, and the lemma will be proved by showing that this last evaluation has a certain chance not to belong to $\trunctreea^{h'}$ for some $h'$ when $\left(\fidtrunctree\right)$ is given to $\pi$.\\
We finally define, for any $h\geq 0$, $c_h=\inf\; \Phi^{-1}([0,\nu\rho^h])$ and $N_h=\left\lfloor\frac{\Lambda}{c_h}\right\rfloor$+1. $c_h$ is the minimum cost necessary to get an evaluation of bias $b\leq \nu\rho^h$, while $N_h$ is the maximum number of such evaluations an algorithm can get, including the last free one. This gives, for any truncated tree $\trunctreea$ and $h\geq0$, $\#\setfidh\leq N_h$

\textit{Reminder:} If $P_{h,i}$ is a cell of $\functionspace$, its children are the K cells $P_{h+1,Ki}$, ..., $P_{h+1,Ki+K-1}$.\\\\

We then prove the two different lower bounds:\\

\textbf{Lower bound a (width limitation)}:\\

This lowerbound is based on the idea that, at a given depth h, the algorithm may have the budget to only open a fraction of the children of the near optimal cells (remember that the idea of opening, previously defined, also takes into account the fidelity level). As the behavior of the algorithm is potentially random, we choose the cell C that is the least likely to be opened as the only one potentially containing the optimum (ie $\trunctreea^{h+1}=C$), and then bound the probability of its opening.\\\\
We first assume that there exists a non-negative integer h such that $\frac{\rho^{-dh}c_h}{\Lambda}\geq K$. Then, there also exists an integer s such that $\frac{\Lambda}{c_h}<K^s\leq \rho^{-dh}$. Because $\rho^{-dh}\geq 1$, we can assume $s$ to be non-negative. Since $K^s$ is an integer, we get that $K^s \geq \left\lfloor\frac{\Lambda}{c_h}\right\rfloor+1=N_h$. We also get that $K^s\leq \rho^{-dh}\leq K^h$ because of $d\leq\frac{log K}{\logrho}$, which implies $h-s\geq0$.\\

We now define the trees we will use, along with Lemma 5, to lower bound the regret:\\
We start with $\trunctree=\left(\bigcup\limits_{h'=0}^{h-s-1}\{\partitioning_{h',0}\}\right)\;\bigcup\; \left(\bigcup\limits_{h'=h-s}^{h}\bigcup\limits_{i=0}^{K^{h'-(h-s)}-1}\{\partitioning_{h',i}\}\right)$\\
We can first notice that $\trunctree$ partially satisfies the hypothesis of lemma 5. Indeed, the first $h-s$ depths of $\trunctree$ only have one node. For $h-s\leq h'\leq h$, there are $K^{h'-(h-s)}$ nodes of depth $h'$. Since $K^s\leq\rho^{-dh}\leq\rho^{-dh'}K^{h-h'}$, we indeed have $K^{h'-(h-s)}\leq \rho^{-dh'}$. However, the theorem also requires at least one node per depth.\\

We now define, for all integer i such that $0\leq i\leq K^{s+1}$,\\
$\trunctree_i=\trunctree\;\bigcup\left(\;\bigcup\limits_{h'=h+1}^{\infty}\{\partitioning_{h',K^{h'-(h+1)}i}\}\right)$\\
We can notice that, for all $i$, $\trunctree_i$ is a truncated tree of $\partitioning$, has at least one node per depth, and the upper bound is still verified, as only one node was added for each empty depth was added. It thus verifies the hypothesis of the lemma.\\\\
We also define $p_i=\mathbb{P}(\exists l\in \setfid^{T_i}_h, X^{\trunctreei}_l\in\partitioning_{h+1,i})$. It corresponds to the probability of the cell $\partitioning_{h',i}$ being opened when the algorithm is given $(f^{\trunctreei}_z)_{z\in Z}$. We then try to upper bound one of the $p_i$
\begin{align*} 
\lsumi p_i&= \lsumi \mathbb{P}\; \left(\bigcup\limits_{l\in\mathbb{N}} \left(l\in \setfid^{\trunctreei}_h, X^{\trunctreei}_l\in\partitioning_{h+1,i}\right)\right)\\
&= \lsumi \mathbb{P}\; \left(\bigsqcup\limits_{l\in\mathbb{N}} \left(l\in \setfid^{\trunctreei}_h, X^{\trunctreei}_l\in\partitioning_{h+1,i}\; \mathrm{and}\; \left(\forall m\in\setfid^{\trunctreei}_h, m<l \implies  X^{\trunctreei}_m \not\in \partitioning_{h+1,i}\right)\right)\right)\\
&= \lsumi\lsuml \;\mathbb{P}\; \left(l\in \setfid^{\trunctreei}_h, X^{\trunctreei}_l\in\partitioning_{h+1,i}\; \mathrm{and}\; \left(\forall m\in\setfid^{\trunctreei}_h, m<l \implies  X^{\trunctreei}_m \not\in \partitioning_{h+1,i}\right)\right)\\
&= \lsumi\lsuml \;\mathbb{P}\; \left(l\in \setfid^{\trunctree}_h, X^{\trunctree}_l\in\partitioning_{h+1,i}\; \mathrm{and}\; \left(\forall m\in\setfid^{\trunctree}_h, m<l \implies  X^{\trunctree}_m \not\in \partitioning_{h+1,i}\right)\right)\\
&\leq \lsumi\lsuml \;\mathbb{P}\; \left(l\in \setfid^{\trunctree}_h, X^{\trunctree}_l\in\partitioning_{h+1,i}\right)\\
&= \;\;\:\lsuml\lsumi\mathbb{P}\; \left(l\in \setfid^{\trunctree}_h, X^{\trunctree}_l\in\partitioning_{h+1,i}\right)\\
&\leq \;\;\:\lsuml\;\mathbb{P}\; \left(l\in \setfid^{\trunctree}_h\right)\\
&=\;\mathbb{E}\left(\# \setfid^{\trunctree}_h\right)\\
&\leq \;N_h\\
&\leq\;K^s
\end{align*}

We therefore have the existence of $i\leq K^{s+1}-1$ such that $p_i\leq\frac{1}{K}\leq\frac{1}{2}$.\\ By giving $\pi$ the fidelity approximations $(f^{\trunctreei}_z)_{z\in Z}$, since $\trunctreei$ respect the hypothesis of Lemma \ref{lem:fun}, we can lower bound the minimax regret with $\frac{1}{2}\nu\rho^h$\\\\

This result needed that $\frac{\rho^{-dh}c_h}{\Lambda}\geq K$ with h a positive integer, which can be rewritten $\frac{\rho^{-dh}\:\inf\Phi^{-1}\left([0,\nu\rho^h]\right)}{K}\geq\Lambda$.\\ Since this condition remains true if we take instead $h'\geq h$, we can just suppose that $\htilde$ is a positive real number such that $\frac{\rho^{-d{\htilde}}\:\inf\Phi^{-1}\left([0,\nu\rho^{\htilde}]\right)}{K}\geq\Lambda$, and in this case get $\regretl\geq\frac{1}{2}\nu\rho^{\lfloor \htilde\rfloor+1}\geq\frac{1}{2}\nu\rho^{\htilde+1}$. By replacing $\htilde$ with $\frac{1}{2}\nu\rho^{\htilde+1}$, this is the same thing as assuming that there exists $r\in]0,\frac{1}{2}\nu\rho]$ such that $\frac{\left(\frac{2r}{\nu\rho}\right)^{-d}\:\inf\Phi^{-1}\left([0,\frac{2r}{\rho}]\right)}{K}\geq \Lambda$ to get $\regretl\geq r$, which concludes.\\\\

\textbf{Lower bound b (depth limitation)}:\\

In this second lower bound, the idea is that, after a depth h, the depth an algorithm can consistently reach when exploring a branch is at most proportional to the number of opening at depth $h$. 
We first take $h$ such that $h\geq N_h$. We define recursively for $h'\geq h$, $(p_{h',i})_{0\leq i\leq K-1}$, $i_{h'}$ and $\trunctree_{h'}$ with

When $h'=h$:
$$\begin{cases}
        p_{h,i}=1\\
        i_h=0\\
        \trunctree_h=\bigcup\limits_{h_p=0}^{h}\{\partitioning_{h_p,0}\}
\end{cases}$$\\\\

When $h'>h$:
$$\begin{cases}
        p_{h',i}=\mathbb{P}\left( \exists l\in \setfid^{\trunctree_{h'-1}}_h,\; X^{\trunctree_{h'-1}}_l\in\partitioning_{h',Ki_{h'-1}+i} \; \mathrm{and}\; \left(\forall m\in\setfid^{\trunctree_{h'-1}}_h, m<l \implies X^{\trunctree_{h'-1}}_m\not\in\partitioning_{h'-1,i_{h'-1}}\right)\right)\\\\
        i_{h'}=Ki_{h'-1}+ \argmin\limits_{i} \;p_{h',i}\\\\
        \trunctree_{h'}=\trunctree_{h'-1}\;\bigcup\;\{\partitioning_{h',i_{h'}}\}
\end{cases}$$\\\\

We also define $\trunctree=\bigcup\limits_{h'=h}^{\infty}\;\trunctree_{h'}$\\
Since $\trunctree$ only has one infinite branch, we have that $\trunctree$ verifies the hypothesis of Lemma 5.

Because, for every depth $h'\geq h+1$, the $p_{h',i}$ are $K$ probability of disjoint events, we necessarily have that the probability associated to $i_{h'}$ is upper bounded by $\frac{1}{K}\leq\frac{1}{2}$.\\\\

We now define, for $h'\geq h+1$ the events $E_{h'}=\left(\exists l\in \setfid^{\trunctree}_h,\; X^{\trunctree}_l\in\partitioning_{h',i_{h'}} \;\mathrm{and}\; \left(\forall m\in\setfid^{\trunctree}_h, m<l \implies X^{\trunctree}_m\not\in\partitioning_{h'-1,i_{h'-1}}\right)\right)$, and the random variable $M=\sum\limits_{h'=h+1}^{4h}\;\mathbbm{1}_{E_{h'}}$
Let's first bound $M$ with a certain probability, using Markov inequality.
\begin{align*} 
\mathbb{E}(M)&=\sum\limits_{h'=h+1}^{4h}\;\mathbb{E}\left(\mathbbm{1}_{E_{h'}}\right)\\
&=\sum\limits_{h'=h+1}^{4h}\;\mathbb{P}\left(\exists l\in \setfid^{\trunctree}_h,\; X^{\trunctree}_l\in\partitioning_{h',i_{h'}} \;\mathrm{and}\; \left(\forall m\in\setfid^{\trunctree}_h, m<l \implies X^{\trunctree}_m\not\in\partitioning_{h'-1,i_{h'-1}}\right)\right)\\
&=\sum\limits_{h'=h+1}^{4h}\;\mathbb{P}\left(\exists l\in \setfid^{\trunctree_{h'-1}}_h,\; X^{\trunctree_{h'-1}}_l\in\partitioning_{h',i_{h'}} \;\mathrm{and}\; \left(\forall m\in\setfid^{\trunctree_{h'-1}}_h, m<l \implies X^{\trunctree_{h'-1}}_m\not\in\partitioning_{h'-1,i_{h'-1}}\right)\right)\\
&=\sum\limits_{h'=h+1}^{4h}\;p_{h',i_{h'}-Ki_{h'-1}}\\
&\leq\sum\limits_{h'=h+1}^{4h}\;\frac{1}{2}\\
&=\;\frac{3h}{2}
\end{align*}
We then get $\mathbb{P}\left(M\geq 2h\right)\leq\frac{\mathbb{E}(M)}{2h}\leq\frac{3}{4}$\\

We now define the event $B=\left(\forall l\in \setfid^{\trunctree}_h,\; X^{\trunctree}_l \not\in \partitioning_{4h,i_{4h}}\right)$ and show that $\left(M<2h\right)\subset B$.

We denote by ${j_1}$, ... , ${j_t}$ the different elements of $\setfid^{\trunctree}_h$ ranked from lowest to highest (with $t=\#\setfid^{\trunctree}_h$), and we define the $(a_r)_{0\leq r\leq t}$ with $a_0=h$ and $a_r=\max\;\{h'\in \llbracket h~;~ 4h \rrbracket\;|\; \exists m\leq r,\; X^{\trunctree}_{j_m}\in\partitioning_{h',i_{h'}}\}\bigcup\left\{h\right\}$ when $r>0$. This definition ensures that $a_r$ is superior or equal to the deepest opened depth after $r$ requests of points at fidelities $z$ of biais $\zeta(z)\leq \nu\rho^h$. We especially have that the sequence $(a_r)_{0\leq r\leq t}$ is non-decreasing. Note that all these values are random variables, and that $\left(a_t<4h\right)$ is exactly the event B.

We can then notice that, with these definitions, for any $h'\in\llbracket h+1~;~ 4h \rrbracket$,
\begin{multline}\left(\exists r\in \llbracket 0~;~ t-1 \rrbracket, X^{\trunctree}_{j_{(r+1)}}\in \partitioning_{h',i_{h'}}\; \mathrm{and}\; \left(\forall m<r+1, X^{\trunctree}_{j_m}\not\in \partitioning_{h'-1,i_{h'-1}}\right)\right)=\\\left(\exists l\in \setfid^{\trunctree}_h,\; X^{\trunctree}_l\in\partitioning_{h',i_{h'}} \;\mathrm{and}\; \left(\forall m\in\setfid^{\trunctree}_h, m<l \implies X^{\trunctree}_m\not\in\partitioning_{h'-1,i_{h'-1}}\right)\right)=E_{h'}\end{multline}

Using this equality, we can see that the existence of $r\in \llbracket 0~;~ t-1 \rrbracket $ such that $h'\in\llbracket a_r+2~;~ a_{r+1} \rrbracket$ implies $E_{h'}$. Counting the $h'$ then let us get that $\#\bigcup\limits_{r=0}^{t-1}\;\llbracket a_r+2~;~ a_{r+1} \rrbracket\leq \sum\limits_{h'=h+1}^{4h}\mathbbm{1}_{E_{h'}}$.

Since all the sets $\llbracket a_r+2~;~ a_{r+1} \rrbracket$ are disjoints (because $(a_r)_{0\leq r\leq t}$ is non-decreasing), we have $\sum\limits_{r=0}^{t-1}\left(a_{r+1}-a_l-1\right)\leq M$, which gives $a_t\leq a_0 + t + M$. As $h\geq N_h$ by hypothesis and $a_0=h$, we have $a_t\leq 2h + M$
Thus, $\left(M<2h\right)$ implies $\left(a_t<4h\right)$, ie $\left(M<2h\right)$ implies B.

Finally, with $\mathbb{P}\left(B\right)\geq \mathbb{P}\left(M < 2h\right)\geq \frac{1}{4}$, we can bound the regret with $\regretl\geq \frac{1}{4}\nu\rho^{4h}$ by giving $\pi$ the fidelity approximations $(f^{\trunctree}_z)_{z\in Z}$, since, as mentioned above, $\trunctree$ respects the hypothesis of Lemma 5.\\\\

To conclude, we do the same thing as for Lower bound a. The hypothesis was the existence of $h$ such that $h\geq N_h$ to get a bound $\nu\rho^{4h}$. It is especially the case when $(h+1)\:\inf\Phi^{-1}\left([0,\nu\rho^{h+1}]\right)\geq\Lambda$. This can similarly be changed to the existence of a real number $\htilde\geq 1$ such that $\htilde\:\inf\Phi^{-1}\left([0,\nu\rho^{\htilde}]\right)\geq\Lambda$ to get a bound $\frac{1}{4}\nu\rho^{4(\htilde+2)}$.\\
With $r=\frac{1}{4}\nu\rho^{(4\htilde+2)}$, we need $r\leq\frac{1}{4}\nu\rho^6$ and $\left(\frac{\log\frac{\nu}{4r}}{4\logrho}-2\right)\:\inf\Phi^{-1}\left(\left[0,\frac{\nu}{\rho}\left(\frac{4r}{\nu}\right)^\frac{1}{4}\right]\right)\geq\Lambda$\\\\
\end{proof}

\begin{proof}[Proof of Theorem \ref{thm:lb}:]
We now apply Lemma \ref{lem:lem6} to get the wanted lower bounds in the different cases.\\

Under \assumsa:\\

We first try to solve $y=\Phi(c)$ for $c\geq1$. We have $y=\frac{A}{c^\alpha}$, which means that $c=\left(\frac{y}{A}\right)^{-\invalpha}$  when $A\geq y$. This implies that, if $r\leq\frac{\rho A}{2}$, $\inf\Phi^{-1}\left([0,\frac{2r}{\rho}]\right)=\left(\frac{2r}{\rho A}\right)^{-\invalpha}$.\\
We thus try, for $r\leq r_{\min}^1\triangleq \min\{\frac{1}{2}\rho A, \frac{1}{2}\nu\rho\}$, to solve the equation 
$\Lambda_r = \frac{1}{K}\left(\frac{2r}{\nu\rho}\right)^{-d}\:\inf\Phi^{-1}\left([0,\frac{2r}{\rho}]\right)$ in order to apply Lemma 6.a. We have $\Lambda_r = \frac{1}{K}\left(\frac{2}{\nu\rho}\right)^{-d}\left(\frac{2}{\rho A}\right)^{-\invalpha}r^{-d-\invalpha}$, which is equivalent to $K\left(\frac{2}{\nu\rho}\right)^{d}\left(\frac{2}{\rho A}\right)^{\invalpha}\Lambda_r = r^{-d-\invalpha}$. With $D_1\triangleq K^\frac{-1}{d+\invalpha}\left(\frac{2}{\nu\rho}\right)^\frac{-d}{d+\invalpha}\left(\frac{2}{\rho A}\right)^{\frac{-1}{1+d\alpha}}$, we have $D_1\Lambda_r^\frac{-1}{d+\invalpha}=r$.\\
Finally, using Lemma 6.a, if $\Lambda\geq\Lambda_{r_{\min}^1}$, we have the bound $\regretl\geq D_1\Lambda^\frac{-1}{d+\invalpha}$. Otherwise, if $\Lambda<\Lambda_{r_{\min}^1}$ we only get that $\regretl\geq r_{\min}^1$.\\

Under \assumsb\\

$\mathbf{d=0:}$ Similarly, we solve $y=\Phi(c)$ for $c\geq 1$. Since $\Phi(c)=Be^{-\frac{c^\beta}{\sigma}}$, we get $c=\left(\sigma \log\left(\frac{B}{y}\right)\right)^\frac{1}{\beta}$ when $B\geq y$. Using this, we have that if $\frac{\nu}{\rho}\left(\frac{4r}{\nu}\right)^\frac{1}{4}\leq B$, ie $r\leq\left(\frac{\rho B}{\nu}\right)^4\frac{\nu}{4}$, then $\inf\Phi^{-1}\left(\left[0,\frac{\nu}{\rho}\left(\frac{4r}{\nu}\right)^\frac{1}{4}\right]\right)=\left(\sigma \log\left(\frac{\rho B}{\nu}\left(\frac{\nu}{4r}\right)^\frac{1}{4}\right)\right)^\frac{1}{\beta}$.

To use Lemma 6.b, we then try to solve, for $r\leq \min\:\{\frac{1}{4}\nu\rho^6,\;\left(\frac{\rho B}{\nu}\right)^4\frac{\nu}{4}\}$, the equation $\Lambda_r= \left(\frac{\log\frac{\nu}{4r}}{4\logrho}-2\right)\:\inf\Phi^{-1}\left(\left[0,\frac{\nu}{\rho}\left(\frac{4r}{\nu}\right)^\frac{1}{4}\right]\right)$. This equation is equivalent to $4^{1+\invbeta}\:\logrho\:\sigma^\frac{-1}{\beta}\Lambda_r=\left(\log\frac{\nu}{4r}-8\logrho\right)\left(\log\left(\left(\frac{B}{\nu}\right)^4\frac{\nu}{4r}\right)\right)^\frac{1}{\beta}$. We define $L_1\triangleq\log\frac{\nu}{4}-8\logrho$, $L_2\triangleq\log\left(\left(\frac{B}{\nu}\right)^4\frac{\nu}{4}\right)$, and $D\triangleq 4^{1+\invbeta}\:\logrho\:\sigma^\frac{-1}{\beta}$ to get it in the form $D\Lambda_r=\left(L_1+\log\frac{1}{r}\right)\left(L_2+\log\frac{1}{r}\right)^\invbeta$. We now also assume $L_1\geq\frac{-1}{2}\log\frac{1}{r}$ and $L_2\geq\frac{-1}{2}\log\frac{1}{r}$ (equivalent to $r\leq e^{2\:L_1}$ and $r\leq e^{2\:L_2}$), to get $D\Lambda_r\geq\left(\frac{1}{2}\log\frac{1}{r}\right)\left(\frac{1}{2}\log\frac{1}{r}\right)^\invbeta$, ie $2D^\frac{\beta}{1+\beta}\Lambda_r^\frac{\beta}{1+\beta}\geq\log\left(\frac{1}{r}\right)$. This gives $e^{-D_2\;\Lambda_r^\frac{\beta}{1+\beta}}\leq r$ with $D_2\triangleq 2D^\frac{\beta}{1+\beta}$.

We now conclude like above, this time using Lemma 6.b. If we define $r_{\min}^2\triangleq \min\{\frac{1}{4}\nu\rho^6,\;\left(\frac{\rho B}{\nu}\right)^4\frac{\nu}{4},\;e^{2\:L_1},\;e^{2\:L_2}\}$, for all $\Lambda\geq\Lambda_{r_{\min}^2}$, we can get the bound $\regretl\geq e^{-D_2\;\Lambda^\frac{\beta}{1+\beta}}$, and $\regretl\geq r_{\min}^2$ otherwise.\\

$\mathbf{d>0:}$  In this case, we simply use that $\inf\Phi^{-1}\left([0,\frac{2r}{\rho}]\right)\geq 1$ by definition of $\Phi$. We try to apply the Lemma 6.a, and solve for $r\leq \frac{1}{2}\nu\rho$,  $\Lambda_r= \frac{1}{K}\left(\frac{2r}{\nu\rho}\right)^{-d}\:\inf\Phi^{-1}\left(\left[0,\frac{\nu}{\rho}\left(\frac{4r}{\nu}\right)^\frac{1}{4}\right]\right)$. We here have $\Lambda_r\geq\frac{1}{K}\left(\frac{2r}{\nu\rho}\right)^{-d}$, which means $\frac{\nu\rho}{2}\:K^\frac{-1}{d}\:\Lambda_r^\frac{-1}{d}\leq r$. With $D_3=\frac{\nu\rho}{2}\:K^\frac{-1}{d}$, we have $D_3\:\Lambda_r^\frac{-1}{d}\leq r$, and we conclude like above with $r_{\min}^3=\frac{1}{2}\nu\rho$.\\

Under \assumsc \\

\textbf{d=0:} We try to apply the Lemma 6.b, and solve, for $r\leq r_{\min}^4=\frac{1}{4}\nu\rho^6$, $\Lambda_r \leq \left(\frac{\log\frac{\nu}{4r}}{4\logrho}-2\right)\:\inf\Phi^{-1}\left(\left[0,\frac{\nu}{\rho}\left(\frac{4r}{\nu}\right)^\frac{1}{4}\right]\right)$. Since the required cost to get any information on the function is $a$, we have that $\inf\Phi^{-1}\left(\left[0,\frac{\nu}{\rho}\left(\frac{4r}{\nu}\right)^\frac{1}{4}\right]\right)=a$, which means that $\Lambda_r=\left(\frac{\log\frac{\nu}{4r}}{4\logrho}-2\right)a$, ie $\log\frac{4}{\nu}+8\logrho+4\logrho\frac{\Lambda_r}{a}=\log\frac{1}{r}$. If we assume $\Lambda_r\geq \frac{a}{4\logrho}\left(\log\frac{4}{\nu}+8\logrho\right)$, we have $8\logrho\frac{\Lambda_r}{a}\geq \log\frac{1}{r}$, and with $D_4=8\logrho\frac{1}{a}$, we finally get $e^{-D_4\:\Lambda_r}\leq r$.

We then conclude like before by applying the Lemma 6.b, and then, for any $\Lambda\geq \max\: \left\{\Lambda_{r_{\min}^4},\;\frac{a}{4\logrho}\left(\log\frac{4}{\nu}+8\logrho\right)\right\}$, we can get the bound $\regretl\geq e^{-D_4\:\Lambda_r}$\\\\

$\mathbf{d>0:}$ The reasoning and the wanted bounds (with $D_3=D_5$) are exactly the same as the case $d>0$ of \assumsb.\\
\end{proof}

\section{Upper bounds}

\begin{proof}[Proof of Proposition \ref{prop:bud}] 

When a cell is opened at fidelity $j$, the maximum budget used for this opening can not exceed $Ke^j$. We can therefore upper bound the budget used for all of the cell opening, ignoring the initial cell opening of $\partitioning_{0,0}$

\begin{align*} 
\sum\limits_{h=1}^{\left\lfloor\hmax\right\rfloor}\sum\limits_{m=1}^{\left\lfloor\frac{\hmax}{h}\right\rfloor}\sum\limits_{j=0}^{\left\lfloor\log\frac{\hmax}{hm}\right\rfloor}Ke^{j} &\leq\sum\limits_{h=1}^{\left\lfloor\hmax\right\rfloor}\sum\limits_{m=1}^{\left\lfloor\frac{\hmax}{h}\right\rfloor}\frac{Ke^{\left\lfloor\log\frac{\hmax}{hm}\right\rfloor+1}}{e-1}\\
&\leq\sum\limits_{h=1}^{\left\lfloor\hmax\right\rfloor}\sum\limits_{m=1}^{\left\lfloor\frac{\hmax}{h}\right\rfloor}\frac{Ke\hmax}{\left(e-1\right)hm}\\
&\leq\sum\limits_{h=1}^{\left\lfloor\hmax\right\rfloor}\frac{Ke\hmax}{\left(e-1\right)h}\left(\log\frac{\hmax}{h}+1\right)\\
&\leq\frac{Ke\hmax}{\left(e-1\right)}\left(\log\hmax+1\right)\left(\log\frac{\hmax}{h}+1\right)\\
&\leq\frac{Ke\hmax}{\left(e-1\right)}\left(\log\Lambda+1\right)^2\\
&\leq \frac{\Lambda}{2}
\end{align*}

The budget used for the initial opening and for the cross-validation can be bounded by
\[\left(K+\log \hmax\right)\hmax\leq K\left(\log\Lambda+1\right)\hmax\leq\frac{\Lambda}{2}\]
This shows that the total budget can be bounded by $\Lambda$.

\end{proof}
\setcounter{lemma}{6}
\begin{lemma}
\label{lem:ub}
Let $\Psi$ be a non-increasing function upper bounding $\Phi$, $j$ a non-negative integer and $\htilde$ a positive real number such that
\begin{enumerate}
  \item  $\Psi(e^j)\leq \nu\rho^{\htilde},$ and
  \item  $\frac{\hmax}{4\htilde e^j}\geq C\rho^{-d{\htilde}}$. 
\end{enumerate}
Then, $\regretl \leq \frac{3\nu}{\rho}\rho^{\htilde} + 2\;\Psi\left(\hmax\right)$.
\end{lemma}

\begin{proof}
We denote by $\partitioning_{\xstar,h'}$ the cell of depth h' containing $\xstar$, $i\star$ the associated index and $h=\lfloor\htilde\rfloor$.\\ 

\paragraph{$\mathbf{Case\;h>0:}$} To prove this lemma, we first want to show that the cell $\partitioning_{\xstar,h}$ is opened with a cost $c_h$ with $c\geq e^{j}$. We do it by induction, and show this is true for any $h'$ such that $h'\leq h$\\
When $h'=0$, the proposition is trivial since there is only one cell of depth 0, opened by the algorithm with $c_0=\hmax\geq e^j$ because of (2) and $\htilde>1$.\\
When $0<h'\leq h$, we assume that $\partitioning_{\xstar,h'-1}$ is opened, and $c_{h'-1}\geq e^j$. We want to show that at least $C\rho^{-dh'}$ cells are opened at a cost superior or equal to $e^j$ and strictly inferior to  $e^{j+1}$. This number $n_{h'}$ of cells is equal to\\
$n_{h'}=\#\{m\in \llbracket 1~;~\frac{\hmax}{h'}  \rrbracket | e^{j+1}>\frac{\hmax}{h'm}\geq e^j\}=\#\{m\in \mathbb{N}^\star | \frac{\hmax}{h'e^{j+1}}<m\leq \frac{\hmax}{h'e^j}\}\geq \frac{\hmax}{h'e^j} - \frac{\hmax}{h'e^{j+1}} -1 \geq \frac{\hmax}{2h'e^j}-1 $\\
Since $\frac{\hmax}{2h'e^j}\geq \frac{\hmax}{2\htilde e^j} \geq 2C\rho^{-d{\htilde}}\geq 2$, we conclude that $n_{h'} \geq \frac{\hmax}{4h'e^j} \geq\frac{\hmax}{4\htilde e^j}\geq C\rho^{-d{\htilde}} \geq C\rho^{-dh'}$.\\
We now suppose that the cell $\partitioning_{\xstar,h'-1}$ was not opened with $c_{h'}\geq e^j$. Then at least $n_{h'}$ cells of depth $h'$ were such that $f_{h',i,j} \geq f_{h',i^\star,j}$, with i denoting the index of any of these cells. This means that, using Assumption 1 and hypothesis (1) of the lemma,

$\sup_{x\in \partitioning_{h',i}} \targetf(x) + \nu\rho^{h'}\geq \targetf(x_{h',i})+\Psi(e^j)\geq f_{z_{e^j}}(x_{h',i})=f_{h',i,j} \geq f_{h',i^\star,j}=f_{z_{e^j}}(x_{h',i^\star})\geq f(x_{h',i^\star}) - \Psi(e^j) \geq f(x_{h',i^\star})-\nu\rho^{h'} \geq \targetf(x^\star) -2\nu\rho^{h'} (\Delta)$ 

absurd since $\ncell_{h'}\leq C\rho^{-dh'} \leq n_{h'}$. We thus have that $\partitioning_{\xstar,h'-1}$ was opened with $c_{h'}\geq e^j$\\\\
We have shown that the cell $\partitioning_{\xstar,h}$ has been opened with a cost $c_h\geq e^j$. This means that the cell has been evaluated at least once for a fidelity of cost $e^{j'}$ with $j'\geq j$. Using the same chain of inequality $(\Delta)$ for h, and $\Psi(e^{j'}) \leq \Psi(e^j)$, we get that the cross-validation candidate $x^c_{j'}$ is such that $\targetf (x^c_j)\geq \targetf (x^\star) - 3\nu\rho^h$. Then, if $x^c_{j''}$ is the returned value, we have that \[\targetf(x^c_{j''}) + \Psi(\hmax) \geq f_{z_{\hmax}} (x^c_{j''}) \geq f_{z_{\hmax}}(x^c_{j'}) \geq \targetf(x^c_{j'}) - \Psi(\hmax) \geq \targetf(x^\star) - \Psi(\hmax) - 3\nu\rho^h\], and we can conclude that $\regretl \leq 3\frac{\nu}{\rho}\rho^{-d{\htilde}}$ since $\htilde \geq h-1$\\
\paragraph{Case h=0}: Since any returned $x^c$ belongs to $\partitioning_{x^\star,0}$, we have that $\regretl \leq \nu\rho^0 \leq \frac{\nu}{\rho}\rho^{\htilde} \leq \frac{3\nu}{\rho}\rho^{\htilde} + 2\;\Psi(\hmax )$
\end{proof}

\begin{proof}[Proof of Theorem \ref{thm:ub}] The proofs for the first two hypothesis are similar since they use the same techniques to apply the lemma\\

\textbf{Under \assumsa} ($\Phi (c)\leq \frac{A}{c^\alpha}$):\\\\
We first try to solve $\frac{\hmax \nu^\invalpha \rho^\frac{h}{\alpha}}{4eA^\invalpha h}=C\rho^{-dh}$ with h unknown. Since this equality is equivalent to $\frac{\hmax \nu^\invalpha\dalpharho}{4CeA^\invalpha h}=\dalpharho h\;e^{\dalpharho h}$, there exists a single positive real number h, which we will name $h_1$, such that the equality is respected, with $h_1=\frac{1}{\dalpharho}W(\frac{\hmax \nu^\invalpha\dalpharho}{4CeA^\invalpha})$. We define the same way $h_2$ as the only solution of the equation $\frac{\hmax}{4h}=C\rho^{-dh}$ (equal to $\frac{\hmax}{4C}$ if $d=0$ and $\frac{1}{d\logrho}W(\frac{\hmax d\logrho}{4C})$ if $d>0$). Note that $h_1$ and $h_2$ are both increasing with the budget. \\
We also define $\Psi$ with $\Psi (x) = \frac{A}{x^\alpha}$ We then discriminate between the two cases:\\

\textbf{First case (High Budget)}: $\nu\rho^{h_1}\leq e^\alpha A$\\
We set $j_1=\lfloor \log(\frac{eA^\invalpha}{\nu^\invalpha \rho^\frac{h_1}{\alpha}})\rfloor$. Thanks to the hypothesis, we have that $j_1$ is a positive integer.\\
We then try to apply Lemma \ref{lem:ub} on $\Psi$, $h_1$ and $j_1$. Since $j_1 > \log(\frac{eA^\invalpha}{\nu^\invalpha \rho^\frac{h_1}{\alpha}})-1 = \log(\frac{A^\invalpha}{\nu^\invalpha \rho^\frac{h_1}{\alpha}})$, we have $\frac{A}{e^{\alpha j_1}} < \nu\rho^{h_1}$ which means $\Psi (e^{j_1}) < \nu\rho^{h_1}$, hypothesis (1).\\
We also have $j_1 \leq \log(\frac{eA^\invalpha}{\nu^\invalpha \rho^\frac{h_1}{\alpha}})$, ie $\frac{1}{e^{j_1}} \geq \frac{\nu^\invalpha \rho^\frac{h_1}{\alpha}}{eA^\invalpha}$, which gives $C\rho^{-dh_1} = \frac{\hmax \nu^\invalpha \rho^\frac{h_1}{\alpha}}{4eA^\invalpha h_1} \leq \frac{\hmax}{4h_1 e^{j_1}}$, hypothesis (2). Applying the lemma then let us obtain the claimed results.

\textbf{Second case (low budget)}: $\nu\rho^{h_1}>e^\alpha A$ \\
We set $j_2=0$ and here try to apply Lemma \ref{lem:ub} on $\Psi$, $h_2$ and $j_2$. We have $h_1\rho^{-dh_1}=\frac{\hmax \nu^\invalpha \rho^\frac{h_1}{\alpha}}{4eA^\invalpha C} > \frac{\hmax}{4C}=h_2\rho^{-dh_2}$, which means $h_1>h_2$ since $h -> h\rho^{-dh}$ is increasing on $\mathbb{R_+}$. This gives $\Psi(e^{j_2}) = A < e^\alpha A < \nu\rho^{h_1} < \nu\rho^{h_2}$ (1) \\
The definition of $h_2$ immediately gives $\frac{\hmax}{4h_2 e^{j_2}}\geq C\rho^{-dh_2}$ (2), and we conclude with the lemma.\\\\

\textbf{Under \assumsb} ($\Phi (c)\leq Be^{\frac{-c^\beta}{\sigma}}$):\\\\

We first define $\abnu$ as equal to $\max(\frac{1}{2\sigma},\log(\frac{B}{\nu}))$. Similarly to assumption 2.a, we then define $h_1$ as the only real positive number h such that $\frac{\hmax}{4he(2\sigma h\logrho )^\invbeta}=C\rho^{-dh}$, and $h_2$ the only real positive number h such that $\frac{\hmax}{4he(2\sigma \abnu)^\invbeta}=C\rho^{-dh}$.\\
\textbf{Value of $h_1$: } We have that $\frac{\hmax}{4Ce(2\sigma \logrho )^\invbeta}=h_1^\frac{1+\beta}{\beta}\rho^{-dh_1}$, which gives $(\frac{\hmax}{4Che})^\frac{\beta}{\beta +1} (\frac{1}{2\sigma \logrho})^\frac{1}{\beta +1}=h_1\rho^{-\frac{\beta}{\beta +1}dh}$. When $d=0$, we then get $h_1=(\frac{\hmax}{4Che})^\frac{\beta}{\beta +1} (\frac{1}{2\sigma \logrho})^\frac{1}{\beta +1}$. When $d>0$, we have $\frac{\beta}{\beta +1}d\logrho (\frac{\hmax}{4Che})^\frac{\beta}{\beta +1} (\frac{1}{2\sigma \logrho})^\frac{1}{\beta +1}=(\frac{\beta}{\beta +1}d\logrho) h_1 e^{\frac{\beta}{\beta +1}d\logrho h_1}$. We finally get $h_1=\frac{\beta +1}{\beta d\logrho}W(\frac{\beta}{\beta +1}d\logrho (\frac{\hmax}{4Che})^\frac{\beta}{\beta +1} (\frac{1}{2\sigma \logrho})^\frac{1}{\beta +1})$\\
\textbf{Value of $h_2$: } The definition gives $\frac{\hmax}{4Ce(2\sigma \abnu)^\invbeta}=h_2\rho^{-dh_2}$. If $d=0$, we have $h_2=\frac{\hmax}{4Ce(2\sigma \abnu)^\invbeta}$. If $d>0$, we get, like for a) $h_2=\frac{1}{d\logrho}W(\frac{\hmax d\logrho}{4Ce(2\sigma \abnu)^\invbeta})$.\\
We also define $\Psi$ with $\Psi (x)=Be^\frac{-x^\beta}{\sigma}$. We again discriminate between the two cases.\\

\textbf{First case (High Budget): } $h_1 \geq \frac{\abnu}{\logrho}$\\
We set $j_1=\lfloor \frac{1}{\beta}log (2\sigma h_1\logrho) \rfloor +1$. Since $2\sigma h_1\logrho \geq 2\sigma \abnu \geq 1$, $j_1\geq 1$\\
We thus try to apply Lemma \ref{lem:ub} on $\Psi$ $h_1$ and $j_1$. Thanks to \[e^{\beta j_1} \geq 2\sigma h_1\logrho \geq \sigma (h_1\logrho + log \frac{B}{\nu})\] we have $Be^\frac{-e^{\beta j_1}}{\sigma} \leq \nu\rho^{h_1}$, ie $\Psi(e^{j_1}) \leq \nu\rho^{h_1}$, hypothesis (1).\\
Since $e^{j_1} \leq e (2\sigma h_1\logrho)^\invbeta$, we have \[\frac{\hmax}{4h_1e^{j_1}}\geq \frac{\hmax}{4h_1e(2\sigma h_1\logrho )^\invbeta}=C\rho^{-dh_1}\;(2)\] 
We then get the wanted results by applying the lemma.\\

\textbf{Second case (Low Budget): } $h_1 < \frac{\abnu}{\logrho}$ \\
We set $j_2=\lfloor \frac{1}{\beta}log (2\sigma \abnu) \rfloor +1 \geq 1$ and apply Lemma \ref{lem:ub} on $\Psi$,$h_2$ and $j_2$.\\
The proof is really similar to the previous case, but we first need to show that $h_1\geq h_2$. This is due to \[h_1\rho^{-dh_1}=\frac{\hmax}{4Ce(2\sigma h_1\logrho)^\invbeta} \geq \frac{\hmax}{4Ce(2\sigma \abnu)^\invbeta}=h_2\rho^{-dh_2}\] which implies $h_1\geq h_2$
We then use this inequality to show \[e^{\beta j_2} \geq 2\sigma \abnu \geq \sigma (h_1\logrho + log \frac{B}{\nu}) \geq \sigma (h_2\logrho + log \frac{B}{\nu})\] and we conclude like before with $\Psi(e^{j_2}) \leq \nu\rho^{h_2}$ (1).\\

Like in the previous case, we have $e^{j_2} \leq e (2\sigma \abnu)^\invbeta$ and thus \[\frac{\hmax}{4h_2e^{j_2}}\geq \frac{\hmax}{4h_2e(2\sigma \abnu )^\invbeta}=C\rho^{-dh_2} \;(2)\] which let us conclude with the lemma.\\

\textbf{Under \assumsc} ($\Phi (c)=0$ for all $c\geq a$)\\\\
We define $h_1$ as the only real solution to the equation $\frac{\hmax}{4aeh}=C\rho^{-dh}$. Similarly to the two previous cases, we get $h_1=\frac{\hmax}{4Cae}$ when d=0 and $h_1=\frac{1}{d\logrho}W(\frac{\hmax d\logrho}{4Cae})$ when $d>0$.\\\\
\textbf{If} $\mathbf{\hmax\geq a:}$\\\\
We now define $\Psi$ with $\Psi (x) = \begin{cases} +\infty & \text{if } x<a \\ 0 & \text{if } x\geq a \end{cases}$, $j_1$ with $j_1=\lfloor \log(ea)\rfloor$ and we try to apply Lemma \ref{lem:ub} on $\Psi$, $h_1$ and $j_1$.\\
Since $e^{j_1}\geq a$, $\Psi (e^{j_1}) \leq \nu\rho^{h_1}$ (1). Using $e^{j_1} \leq ea$, we have \[\frac{\hmax}{4h_1e^{j_1}}\geq \frac{\hmax}{4h_1ea}=C\rho^{-dh_1} (2)\]. We then use the lemma (and $\Psi (\hmax)=0$, but in its proof, since we can upper bound $\Psi(e^{j'})$ by 0 in this specific case, we can loose a 3 factor in the result and just keep $\frac{\nu}{\rho}\rho^{h_1}$\\
\textbf{If} $\mathbf{\hmax< a:}$\\
We then have $h_1<1$. Since $\regretl\leq \nu$, as explained in the lemma proof, we have $\regretl\leq\frac{\nu}{\rho}\rho^{h_1}$ which concludes.
\end{proof}

\begin{proof} [Proof of Corollary \ref{co:ub}]
For the proof of the different bounds, we use that $W(x)\geq \log\,x\,-\, \log\, \log\, x$ for $x\geq e$, as shown in \cite{hoorfar2008inequalities}. We can also notice that for $x\leq e$, we have $W(x)\geq x/e$.\\
We remind that the value of $\hmax$ is bounded by \[\frac{(e-1)\Lambda}{2Ke\left(\log\ \Lambda+1\right)^2}\]
\textbf{Under \assumsa} ($\Phi (c)\leq \frac{A}{c^\alpha}$):\\\\

\textbf{High Budget:} We define $\lambdac=\frac{\hmax \nu^\invalpha\dalpharho}{4CeA^\invalpha}$.\\ If $\lambdac\geq e$ then

\begin{align*} 
\regretl&\leq \frac{3\nu}{\rho}\rho^{h_1} +2\frac{A}{\hmax^\alpha}\\
&=\frac{3\nu}{\rho}\rho^{\frac{1}{\dalpharho}W(\lambdac)} +2\frac{A}{\hmax^\alpha}\\
&\leq\frac{3\nu}{\rho}\rho^{\frac{1}{\dalpharho}\log\left(\frac{\lambdac}{\log\lambdac}\right)} +2\frac{A}{\hmax^\alpha}\\
&=\frac{3\nu}{\rho}\left(\frac{\lambdac}{\log\lambdac}\right)^{\frac{-1}{d+\invalpha}} +2\frac{A}{\hmax^\alpha}\\
\end{align*}

If $\lambdac\leq e$, then  

\begin{align*} 
\regretl&\leq \frac{3\nu}{\rho}\rho^{h_1} +2\frac{A}{\hmax^\alpha}\\
&=\frac{3\nu}{\rho}\rho^{\frac{1}{\dalpharho}W(\lambdac)} +2\frac{A}{\hmax^\alpha}\\
&=\frac{3\nu}{\rho}\rho^{\frac{1}{\dalpharho}\frac{\lambdac}{e}} +2\frac{A}{\hmax^\alpha}\\
\end{align*}

\textbf{Low Budget:} If $d=0$, the exponential upper bound is given by 
\[\regretl\leq\frac{3\nu}{\rho}\rho^{h_2} +2\frac{A}{\hmax^\alpha}\leq \frac{3\nu}{\rho}\rho^{\frac{\hmax}{4C}} +2\frac{A}{\hmax^\alpha}\]

If $d>0$, we define $\lambdac=\frac{\hmax d\logrho}{4C}$\\
Then, if $\lambdac\geq e$

\begin{align*} 
\regretl&\leq \frac{3\nu}{\rho}\rho^{h_1} +2\frac{A}{\hmax^\alpha}\\
&=\frac{3\nu}{\rho}\rho^{\frac{1}{d\logrho}W(\lambdac)} +2\frac{A}{\hmax^\alpha}\\
&\leq\frac{3\nu}{\rho}\rho^{\frac{1}{d\logrho}\log\left(\frac{\lambdac}{\log\lambdac}\right)} +2\frac{A}{\hmax^\alpha}\\
&=\frac{3\nu}{\rho}\left(\frac{\lambdac}{\log\lambdac}\right)^{\frac{-1}{d}} +2\frac{A}{\hmax^\alpha}\\
\end{align*}

Else, if $\lambdac<e$

\begin{align*} 
\regretl&\leq \frac{3\nu}{\rho}\rho^{h_1} +2\frac{A}{\hmax^\alpha}\\
&=\frac{3\nu}{\rho}\rho^{\frac{1}{d\logrho}W(\lambdac)} +2\frac{A}{\hmax^\alpha}\\
&\leq\frac{3\nu}{\rho}\rho^{\frac{1}{d\logrho}\frac{\lambdac}{e}} +2\frac{A}{\hmax^\alpha}\\
\end{align*}

\textbf{Under \assumsb} ($\Phi (c)\leq Be^{\frac{-c^\beta}{\sigma}}$):\\\\
\textbf{High Budget:} If $d=0$, then we directly get 
\[\regretl \leq \frac{3\nu}{\rho}\rho^{h_1} + 2Be^{\frac{-\hmax^\beta}{\sigma}}\leq \frac{3\nu}{\rho}\rho^{(\frac{\hmax}{4Che})^\frac{\beta}{\beta +1} (\frac{1}{2\sigma \logrho})^\frac{1}{\beta +1}} + 2Be^{\frac{-\hmax^\beta}{\sigma}}\]

If $d>0$, we define $\lambdac=\left(\frac{\beta}{\beta +1}d\logrho\left(\frac{1}{2\sigma \logrho}\right)^\frac{1}{\beta +1}\right)^\frac{\beta+1}{\beta}\frac{\hmax}{4Che}$\\
Then, if $\lambdac\geq e$
\begin{align*} 
\regretl&\leq \frac{3\nu}{\rho}\rho^{h_1}+2Be^{\frac{-\hmax^\beta}{\sigma}}\\
&=\frac{3\nu}{\rho}\rho^{\frac{\beta +1}{\beta d\logrho}W(\lambdac^\frac{\beta}{\beta+1})} +2Be^{\frac{-\hmax^\beta}{\sigma}}\\
&\leq\frac{3\nu}{\rho}\rho^{\frac{\beta +1}{\beta d\logrho}\log\left(\frac{\lambdac^\frac{\beta}{\beta+1}}{\log\lambdac}\right)} +2Be^{\frac{-\hmax^\beta}{\sigma}}\\
&=\frac{3\nu}{\rho}\left(\frac{\lambdac}{\left(\log\lambdac\right)^\frac{\beta+1}{\beta}}\right)^{\frac{-1}{d}} +2Be^{\frac{-\hmax^\beta}{\sigma}}\\
\end{align*}

Else, if $\lambdac< e$

\begin{align*} 
\regretl&\leq \frac{3\nu}{\rho}\rho^{h_1}+2Be^{\frac{-\hmax^\beta}{\sigma}}\\
&=\frac{3\nu}{\rho}\rho^{\frac{\beta +1}{\beta d\logrho}W(\lambdac^\frac{\beta}{\beta+1})} +2Be^{\frac{-\hmax^\beta}{\sigma}}\\
&\leq\frac{3\nu}{\rho}\rho^{\frac{\beta +1}{\beta d\logrho}\frac{\lambdac^\frac{\beta}{\beta+1}}{e}} +2Be^{\frac{-\hmax^\beta}{\sigma}}\\
\end{align*}

\textbf{Low Budget:} If $d=0$, we get 

\[\regretl \leq \frac{3\nu}{\rho}\rho^{h_1} + 2Be^{\frac{-\hmax^\beta}{\sigma}}\leq \frac{3\nu}{\rho}\rho^{\frac{\hmax}{4Ce(2\sigma \abnu)^\invbeta}} + 2Be^{\frac{-\hmax^\beta}{\sigma}}\]

If $d>0$, we define $\lambdac=\frac{\beta}{\beta +1}d\logrho (\frac{\hmax}{4Che})^\frac{\beta}{\beta +1} (\frac{1}{2\sigma \logrho})^\frac{1}{\beta +1}$\\
Then, if $\lambdac\geq e$

\begin{align*} 
\regretl&\leq \frac{3\nu}{\rho}\rho^{h_1} +2Be^{\frac{-\hmax^\beta}{\sigma}}\\
&=\frac{3\nu}{\rho}\rho^{\frac{1}{d\logrho}W(\lambdac)} +2Be^{\frac{-\hmax^\beta}{\sigma}}\\
&\leq\frac{3\nu}{\rho}\rho^{\frac{1}{d\logrho}\log\left(\frac{\lambdac}{\log\lambdac}\right)} +2Be^{\frac{-\hmax^\beta}{\sigma}}\\
&=\frac{3\nu}{\rho}\left(\frac{\lambdac}{\log\lambdac}\right)^{\frac{-1}{d}} +2Be^{\frac{-\hmax^\beta}{\sigma}}\\
\end{align*}

Else, if $\lambdac< e$

\begin{align*} 
\regretl&\leq \frac{3\nu}{\rho}\rho^{h_1} +2Be^{\frac{-\hmax^\beta}{\sigma}}\\
&=\frac{3\nu}{\rho}\rho^{\frac{1}{d\logrho}W(\lambdac)} +2Be^{\frac{-\hmax^\beta}{\sigma}}\\
&\leq\frac{3\nu}{\rho}\rho^{\frac{1}{d\logrho}\frac{\lambdac}{e}} +2Be^{\frac{-\hmax^\beta}{\sigma}}\\
\end{align*}

\textbf{Under \assumsc} ($\Phi (c)=0$ for all $c\geq a$)

If $d=0$, we get
\[\regretl \leq \frac{\nu}{\rho}\rho^{h_1}\leq \frac{\nu}{\rho}\rho^{\frac{\hmax}{4Cae}}\]

If $d>0$, we define $\lambdac=\frac{\hmax d\logrho}{4Cae}$\\
Then, if $\lambdac\geq e$

\begin{align*} 
\regretl&\leq \frac{\nu}{\rho}\rho^{h_1}\\
&=\frac{\nu}{\rho}\rho^{\frac{1}{d\logrho}W(\lambdac)}\\
&\leq\frac{\nu}{\rho}\rho^{\frac{1}{d\logrho}\log\left(\frac{\lambdac}{\log\lambdac}\right)}\\
&=\frac{\nu}{\rho}\left(\frac{\lambdac}{\log\lambdac}\right)^{\frac{-1}{d}}\\
\end{align*}

Else, if $\lambdac< e$
\begin{align*} 
\regretl&\leq \frac{\nu}{\rho}\rho^{h_1}\\
&=\frac{\nu}{\rho}\rho^{\frac{1}{d\logrho}W(\lambdac)}\\
&\leq\frac{\nu}{\rho}\rho^{\frac{1}{d\logrho}\frac{\lambdac}{e}}\\
\end{align*}

\end{proof}

\end{document}